%% file: example_paper.tex
\documentclass{article}

\usepackage{microtype}
\usepackage{graphicx}
\usepackage{subcaption}
\usepackage{booktabs} % for professional tables
\usepackage{multirow} % For combining rows (Task names)

\usepackage{hyperref}

\usepackage[preprint]{icml2026}

\usepackage{amsmath}
\usepackage{amssymb}
\usepackage{mathtools}
\usepackage{amsthm}
\usepackage{tikz}
\usetikzlibrary{shapes, arrows.meta, positioning, shadows, fit, backgrounds, calc}
\usepackage{xcolor}
\definecolor{safegreen}{RGB}{0, 150, 0}   % Darker green for readability
\definecolor{unsafegray}{gray}{0.6}       % Light gray for violations

\newcommand{\best}[1]{\textcolor{safegreen}{\textbf{#1}}} % Green + Bold (Optimal Safe)
\newcommand{\bad}[1]{\textcolor{unsafegray}{#1}}         % Gray (Unsafe)

\usepackage[capitalize,noabbrev]{cleveref}


\theoremstyle{plain}
\newtheorem{theorem}{Theorem}[section]

\newtheorem{lemma}[theorem]{Lemma}
\newtheorem{corollary}[theorem]{Corollary}
\theoremstyle{definition}

\theoremstyle{remark}

\usepackage[textsize=tiny]{todonotes}

\icmltitlerunning{Submission and Formatting Instructions for ICML 2026}

\begin{document}

\twocolumn[
  \icmltitle{Safe Langevin Soft Actor Critic}

  % It is OKAY to include author information, even for blind submissions: the
  % style file will automatically remove it for you unless you've provided
  % the [accepted] option to the icml2026 package.

  % List of affiliations: The first argument should be a (short) identifier you
  % will use later to specify author affiliations Academic affiliations
  % should list Department, University, City, Region, Country Industry
  % affiliations should list Company, City, Region, Country

  % You can specify symbols, otherwise they are numbered in order. Ideally, you
  % should not use this facility. Affiliations will be numbered in order of
  % appearance and this is the preferred way.
  \icmlsetsymbol{equal}{*}

  \begin{icmlauthorlist}
    \icmlauthor{Mahesh Keswani}{equal,yyy}
    \icmlauthor{Samyak Jain}{comp}
    \icmlauthor{Raunak P. Bhattacharyya}{yyy}
    %\icmlauthor{}{sch}
    %\icmlauthor{}{sch}
  \end{icmlauthorlist}

  \icmlaffiliation{yyy}{Yardi School of Artificial Intelligence, IIT Delhi, India}
  \icmlaffiliation{comp}{Department of Mathematics, IIT Delhi, India}

  \icmlcorrespondingauthor{Mahesh Keswani}{aiy247544@iitd.ac.in}
  \icmlcorrespondingauthor{Samyak Jain}{mt1221658@maths.iitd.ac.in}
  \icmlcorrespondingauthor{Raunak P. Bhattacharyya}{raunakbh@iitd.ac.in}

  % You may provide any keywords that you find helpful for describing your
  % paper; these are used to populate the "keywords" metadata in the PDF but
  % will not be shown in the document
  \icmlkeywords{Machine Learning, ICML}

  \vskip 0.3in
]

% this must go after the closing bracket ] following \twocolumn[ ...

% This command actually creates the footnote in the first column listing the
% affiliations and the copyright notice. The command takes one argument, which
% is text to display at the start of the footnote. The \icmlEqualContribution
% command is standard text for equal contribution. Remove it (just {}) if you
% do not need this facility.

% Use ONE of the following lines. DO NOT remove the command.
% If you have no special notice, KEEP empty braces:
\printAffiliationsAndNotice{}  % no special notice (required even if empty)
% Or, if applicable, use the standard equal contribution text:
% \printAffiliationsAndNotice{\icmlEqualContribution}

\begin{abstract}
  % Safe reinforcement learning (SafeRL) is essential for deploying autonomous agents in safety-critical applications such as robotics and autonomous driving. However, existing SafeRL methods face several challenges such as value-based approaches often converge to sharp local minima leading to brittle policies, Lagrangian methods that optimize for average constraint satisfaction may fail to account for tail risks, and single-critic architectures provide limited robustness in value estimation. We propose ABC, a SafeRL algorithm that addresses these limitations through three key contributions. First, we employ an ensemble of bootstrapped twin reward critics trained with adaptive Stochastic Gradient Langevin Dynamics (aSGLD) to encourage exploration of flat minima and improve generalization. Second, we model the cost return distribution and optimize a Conditional Value-at-Risk (CVaR) constrained objective to explicitly mitigate tail risks. Third, we adapt the Lagrange multiplier based on the CVaR of recent episodic costs rather than their mean, ensuring conservative constraint enforcement. Experiments on Safety-Gymnasium continuous control tasks demonstrate that ABC achieves competitive or superior performance compared to state-of-the-art off-policy SafeRL baselines while maintaining robust safety across all environments.

      Balancing reward and safety in constrained reinforcement learning remains challenging due to poor generalization from sharp value minima and inadequate handling of heavy-tailed risk distribution. We introduce Safe Langevin Soft Actor-Critic (SL-SAC), a principled algorithm that addresses both issues through parameter-space exploration and distributional risk control. Our approach combines three key mechanisms: (1) Adaptive Stochastic Gradient Langevin Dynamics (aSGLD) for reward critics, promoting ensemble diversity and escape from poor optima; (2) distributional cost estimation via Implicit Quantile Networks (IQN) with Conditional Value-at-Risk (CVaR) optimization for tail-risk mitigation; and (3) a reactive Lagrangian relaxation scheme that adapts constraint enforcement based on the empirical CVaR of episodic costs. We provide theoretical guarantees on CVaR estimation error and demonstrate that CVaR-based Lagrange updates yield stronger constraint violation signals than expected-cost updates. On Safety-Gymnasium benchmarks, SL-SAC achieves the lowest cost in 7 out of 10 tasks while maintaining competitive returns, with cost reductions of 19-63\% in velocity tasks compared to state-of-the-art baselines.
\end{abstract}

\input{introduction}
\input{preliminaries}
\input{method}
\input{experiments}
\input{related_work}
\input{conclusion}

\section*{Impact Statement}
This paper introduces SL-SAC, a framework designed to improve constraint satisfaction in Safe Reinforcement Learning. By incorporating distributional cost modeling and Langevin dynamics, our work aims to mitigate tail risks and enhance value estimation robustness. These contributions are relevant for developing reliable agents in safety-critical domains such as autonomous driving and robotics. There are no specific negative societal consequences we feel must be highlighted here

% In the unusual situation where you want a paper to appear in the
% references without citing it in the main text, use \nocite
% \nocite{langley00}

\bibliography{references}
\bibliographystyle{icml2026}

%%%%%%%%%%%%%%%%%%%%%%%%%%%%%%%%%%%%%%%%%%%%%%%%%%%%%%%%%%%%%%%%%%%%%%%%%%%%%%%
%%%%%%%%%%%%%%%%%%%%%%%%%%%%%%%%%%%%%%%%%%%%%%%%%%%%%%%%%%%%%%%%%%%%%%%%%%%%%%%
% APPENDIX
%%%%%%%%%%%%%%%%%%%%%%%%%%%%%%%%%%%%%%%%%%%%%%%%%%%%%%%%%%%%%%%%%%%%%%%%%%%%%%%
%%%%%%%%%%%%%%%%%%%%%%%%%%%%%%%%%%%%%%%%%%%%%%%%%%%%%%%%%%%%%%%%%%%%%%%%%%%%%%%
\newpage
\appendix
\onecolumn
\input{appendix}

\end{document}

%% file: introduction.tex
\section{Introduction}

Reinforcement learning (RL) has achieved remarkable success across diverse domains, from mastering complex games~\citep{silver2016mastering, mnih2013playing} to controlling robotic systems~\citep{levine2016end, dmitry2018qt}. However, the trial-and-error nature of RL poses significant challenges for safety-critical applications such as autonomous driving, healthcare, and industrial control, where constraint violations can lead to catastrophic consequences. Safe reinforcement learning (SafeRL) addresses this challenge by incorporating explicit safety constraints into the learning process, enabling agents to maximize cumulative rewards while satisfying cost constraints that encode safety requirements~\citep{achiam2017constrained}.

Despite significant progress in SafeRL, achieving robust constraint satisfaction in stochastic environments remains an open challenge due to two fundamental limitations. First, the predominant reliance on expected cost constraints fails to account for tail risks; by averaging costs over trajectories, these methods effectively mask rare but catastrophic failures, leaving agents vulnerable to worst-case scenarios~\citep{achiam2017constrained, ray2019benchmarking}. Second, standard off-policy algorithms are limited by the brittleness of value estimation. Relying on deterministic point estimates neglects the epistemic uncertainty required for safe exploration, while conventional first-order optimizers (e.g., Adam~\citep{adam2014method}) tend to converge to sharp local minima~\citep{ishfaq2025langevin}. 

% Despite significant progress in SafeRL, several challenges remain in ensuring robust constraint satisfaction in practice. First, many approaches rely on expected cost estimates for constraint enforcement, which may underestimate tail risks when rare but severe constraint violations are averaged out by frequent low-cost episodes~\citep{achiam2017constrained, ray2019benchmarking}. Second, standard actor-critic methods often employ a single reward critic, leading to high-variance off-policy value estimates that can be mitigated by ensemble and twin-critic architectures, while widely used gradient-based optimizers such as Adam tend to converge to sharp minima that limit the generalization of learned Q-values~\citep{chen2021randomized, osband2016deep, fujimoto2018addressing, adam2014method}.

To address these challenges, we propose Safe Langevin Soft Actor Critic (SL-SAC), a SafeRL algorithm that integrates distributional cost modeling, ensemble-based reward value estimation with parameter-space exploration, and risk-sensitive constraint enforcement. SL-SAC tackles tail-risk underestimation by employing a distributional cost critic trained via Implicit Quantile Networks (IQN)~\citep{dabney2018implicit}, enabling Conditional Value-at-Risk (CVaR)~\citep{rockafellar2000optimization} based constraint enforcement that explicitly accounts for worst-case cost scenarios. Second, for robust reward estimation, SL-SAC uses an ensemble of twin critics~\citep{fujimoto2018addressing} optimized via adaptive Stochastic Gradient Langevin Dynamics (aSGLD)~\citep{ishfaq2025langevin}. The aSGLD optimizer injects noise into gradient updates, encouraging parameter diversity and helping the ensemble escape sharp minima, which improves generalization without requiring full Bayesian inference. In addition, SL-SAC introduces a risk-sensitive Lagrangian that dynamically adjusts constraint enforcement based on the CVaR of recent episodic costs.

Theoretically, we establish two key results that justify our design choices: (1) We bound the CVaR estimation error in terms of the quantile regression error, demonstrating that accurate quantile estimation directly translates to reliable tail-risk assessment; (2) We prove that CVaR-based Lagrange updates provide stronger constraint violation signals than expected-cost updates, leading to tighter control over tail risks. We evaluate SL-SAC on the Safety-Gymnasium benchmark~\citep{ji2023safety} and show it consistently outperforms state-of-the-art baselines, achieving superior reward-cost trade-offs across diverse continuous-control tasks (Table~\ref{tab:main_results}). To further assess generalization, we test SL-SAC on autonomous driving in the MetaDrive simulator~\citep{li2022metadrive}, where it achieves lower costs and superior returns than baselines in stochastic traffic environments. Ablation studies validate that both distributional cost modeling and parameter-space exploration via aSGLD are crucial for robust safety and reward value estimation.

%% file: preliminaries.tex
\section{Background}

\subsection{Constrained Markov Decision Processes (CMDPs)}

We formulate the SafeRL problem as a Constrained Markov Decision Process (CMDP)~\citep{altman2021constrained}, which extends the standard MDP framework by introducing explicit safety constraints. A CMDP is defined by the tuple $(\mathcal{S}, \mathcal{A}, P, r, c, d_0, \gamma, \beta)$, where $\mathcal{S}$ and $\mathcal{A}$ denote the state and action spaces, $P: \mathcal{S} \times \mathcal{A} \to \Delta(\mathcal{S})$ represents the transition dynamics, $d_0$ is the initial state distribution, and $\gamma \in [0,1)$ is the discount factor. At each timestep, the agent receives both a reward $r(s,a)$ to maximize and a cost $c(s,a)$ to constrain.

The objective is to find a policy $\pi^*$ that maximizes the expected discounted return $J(\pi) = \mathbb{E}_{\tau \sim \pi}[\sum_{t=0}^\infty \gamma^t r(s_t, a_t)]$ while ensuring the expected cumulative cost $J_C(\pi) = \mathbb{E}_{\tau \sim \pi}[\sum_{t=0}^\infty \gamma^t c(s_t, a_t)]$ remains below safety threshold $\beta$. While standard CMDPs define $J_C(\pi)$ as the expected cost, risk-constrained formulations may replace this with risk measures such as CVaR to account for tail events:
\begin{equation*}
    \pi^* = \operatorname*{argmax}_{\pi} J(\pi) \quad \text{s.t.} \quad J_C(\pi) \leq \beta,
\end{equation*}
where $\tau = (s_0, a_0, s_1, a_1, \ldots)$ denotes a trajectory with $s_0 \sim d_0$ and $a_t \sim \pi(\cdot|s_t)$. 

\subsection{Maximum Entropy Reinforcement Learning}

To enhance exploration in constrained optimization, we adopt the Maximum Entropy RL framework~\citep{ziebart2010modeling, haarnoja2018soft}. Unlike standard RL that maximizes expected reward alone, MaxEnt RL augments the standard objective with entropy regularization to encourage broad exploration, simultaneously maximizing expected return and policy entropy:
\begin{equation*}
    J_{\text{MaxEnt}}(\pi) = \mathbb{E}_{\tau \sim \pi} \bigg[ \sum_{t=0}^\infty \gamma^t \Big( r(s_t, a_t) + \alpha \mathcal{H}(\pi(\cdot | s_t)) \Big) \bigg],
\end{equation*}
where $\alpha > 0$ is the temperature parameter regulating the stochasticity of the policy, and $\mathcal{H}(\pi(\cdot|s)) = -\mathbb{E}_{a \sim \pi(\cdot|s)}[\log \pi(a|s)]$ is the entropy of the policy.

By combining the CMDP formulation with the MaxEnt objective, we can employ Lagrangian relaxation to convert the constrained problem into an unconstrained max-min formulation. This formulation balances reward maximization, entropy-induced exploration, and constraint satisfaction:
\begin{equation*}
    \max_{\pi} \min_{\lambda \geq 0} \mathcal{L}(\pi, \lambda) = J_{\text{MaxEnt}}(\pi) - \lambda (J_C(\pi) - \beta),
\end{equation*}
where $\lambda$ is the Lagrange multiplier enforcing the cost constraint. $\lambda$ is adjusted via projected gradient ascent to penalize constraint violations: 
\begin{equation*}
    \lambda \leftarrow \max(0, \lambda + \eta_\lambda (J_C(\pi) - \beta)),
\end{equation*}
where $\eta_\lambda > 0$ is the learning rate for the multiplier.

\begin{figure*}[t]
    \centering
    \includegraphics[width=\linewidth]{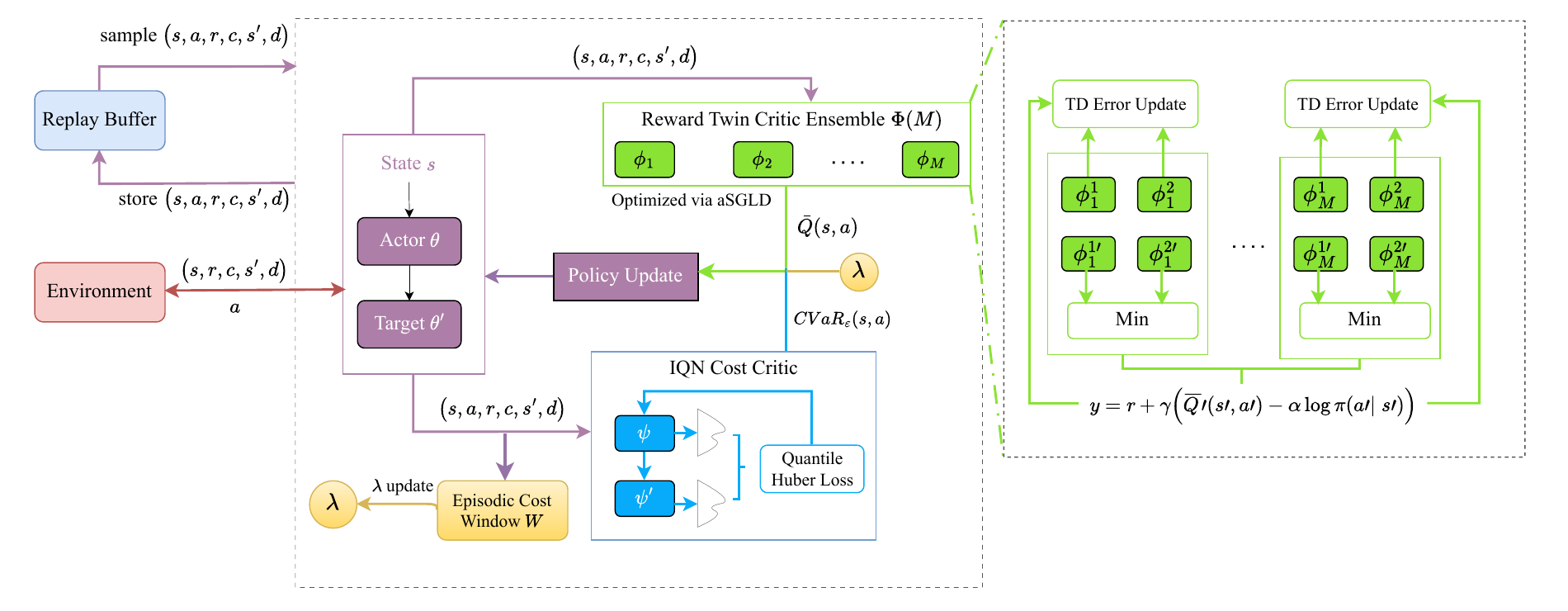}
    \caption{SL-SAC architecture overview. The policy interacts with the environment, storing transitions in replay memory. An ensemble of reward critics is trained via aSGLD, while a distributional cost critic learns quantiles via IQN. The policy optimizes a Lagrangian objective balancing reward, entropy, and CVaR-based safety, with $\lambda$ updated using empirical CVaR of recent episodes.}
    \label{fig:abc_arch}
\end{figure*}

\subsection{Langevin Monte Carlo}

While MaxEnt RL addresses the stochasticity of the policy, reliable value estimation in SafeRL requires avoiding convergence to sharp local minima that can lead to poor generalization. The Langevin Monte Carlo (LMC), a class of Markov Chain Monte Carlo (MCMC) sampling algorithms, addresses this by injecting noise into gradient updates, promoting exploration in parameter space and helping escape suboptimal regions. LMC targets distributions of the form:
\begin{equation*}
    p^*(\theta) \propto \exp\left(-\frac{1}{T} U(\theta)\right),
\end{equation*}
where \( U: \mathbb{R}^n \to \mathbb{R} \) is a continuously differentiable energy function, \( \theta \in \mathbb{R}^n \) are model parameters, and \( T > 0 \) is a temperature parameter~\citep{roberts1996exponential}. The classical LMC algorithm is based on the discretization of the overdamped Langevin diffusion process~\citep{rossky1978brownian, roberts2002langevin}:
\begin{equation*}
    d\theta_t = -\nabla U(\theta_t)\,dt + \sqrt{2T}\,dB_t,
\end{equation*}
where \( B_t \) is standard Brownian motion in \( \mathbb{R}^n \). Under appropriate conditions on \( U \), this process admits \( p^*(\theta) \) as its stationary distribution.

In the context of deep learning, the continuous-time dynamics are discretized using the Euler–Maruyama scheme. When the exact gradient \( \nabla U(\theta) \) is replaced by a stochastic mini-batch estimate, the method becomes Stochastic Gradient Langevin Dynamics (SGLD)~\citep{welling2011bayesian}:
\begin{equation*}
    \theta_{k+1} = \theta_k - \eta \tilde{\nabla} U(\theta_k) + \sqrt{2T\eta} \, \xi_k,
\end{equation*}
where \( \eta > 0 \) is the step size and \( \xi_k \sim \mathcal{N}(0, I_n) \) is injected Gaussian noise.

In Deep RL, SGLD enables Bayesian posterior approximation of value functions rather than relying on point estimates. Its noise injection facilitates escaping sharp minima to improve generalization~\citep{ishfaq2025langevin}, while driving parameter diversity across ensembles to yield robust uncertainty estimates and mitigate mode collapse.

% \subsection{Conditional Value-at-Risk}

% Conditional Value-at-Risk (CVaR), also known as Expected Shortfall, is a coherent risk measure that quantifies tail risk in a distribution~\citep{rockafellar2000optimization}. For a random variable $Z$ with cumulative distribution function $F_Z$, the CVaR at confidence level $\epsilon \in (0,1]$ is defined as:
% \begin{equation}
%     \text{CVaR}_\epsilon[Z] = \mathbb{E}\left[Z \mid Z \geq F^{-1}_Z(1-\epsilon)\right] = \frac{1}{\epsilon}\int_{1-\epsilon}^{1} F^{-1}_Z(\tau) \, d\tau,
%     \label{eq:cvar_definition}
% \end{equation}
% where $F^{-1}_Z$ is the quantile function. Unlike the expectation, which treats all outcomes equally, CVaR focuses on the worst $\epsilon$-fraction of the distribution. Smaller values of $\epsilon$ correspond to more risk-averse behavior by emphasizing rarer, more severe outcomes, while $\epsilon = 1$ reduces to the expected value (risk-neutral case).

% In SafeRL, CVaR can be used to enforce constraints on the tail behavior of the cost return distribution rather than its mean. Given a cost return distribution $Z^C(s,a)$ and risk level $\epsilon$, a CVaR-based safety constraint requires:
% \begin{equation}
%     \text{CVaR}_\epsilon[Z^C_\pi] \leq \beta,
%     \label{eq:cvar_constraint}
% \end{equation}
% where $Z^C_\pi$ denotes the cost return under policy $\pi$. 

\subsection{Conditional Value-at-Risk (CVaR)}

Conditional Value-at-Risk (CVaR), also known as Expected Shortfall, is a coherent risk measure that quantifies the tail risk of a distribution. For a random variable $Z$ with cumulative distribution function $F_Z$, the CVaR at confidence level $\epsilon \in (0,1)$ is defined as the expected value of the worst-case $\epsilon$-portion of outcomes:
\begin{equation*}
    \text{CVaR}_\epsilon[Z] = \mathbb{E}\left[Z \mid Z \geq F^{-1}_Z(1-\epsilon)\right] = \frac{1}{\epsilon}\int_{1-\epsilon}^{1} F^{-1}_Z(\tau) \, d\tau,
    \label{eq:cvar_definition}
\end{equation*}
where $F^{-1}_Z$ is the quantile function. Unlike the expected value, which averages over the entire distribution, CVaR focuses exclusively on the tail, making it sensitive to extreme outcomes and heavy tails. 

In SafeRL, CVaR can be used to enforce constraints on the tail behavior of the cost return distribution rather than its expected value. Given a cost return distribution $Z^C(s,a)$ and risk level $\epsilon$, a CVaR-based safety constraint requires:
\begin{equation*}
    \text{CVaR}_\epsilon[Z^C_\pi] \leq \beta,
    \label{eq:cvar_constraint}
\end{equation*}
where $Z^C_\pi$ denotes the cost return under policy $\pi$. 

% In SafeRL, CVaR is commonly employed to enforce constraints on the tail of the cost distribution $Z^C_\pi$, ensuring safety even in worst-case scenarios:
% \begin{equation}
%     \text{CVaR}_\eps[Z^C_\pi] \leq \beta.
%     \label{eq:cvar_constraint}
% \end{equation}

% However, optimizing or constraining CVaR directly is computationally challenging. Standard sample-based estimators often rely on discarding the $(1-\alpha)$-portion of "safe" or "high-return" trajectories to isolate the tail, resulting in severe **sample inefficiency**~\citep{mead2025return}. As noted by \citet{mead2025return}, this approach also suffers from \textit{"blindness to success"}, where the optimizer fails to learn from non-tail trajectories that successfully avoid risks. While recent works like Return Capping propose modifying the objective to utilize all data~\citep{mead2025return}, in this work we address these estimation challenges by employing a distributional critic that learns the full quantile function of the costs, allowing for sample-efficient and differentiable CVaR estimation.

%% file: method.tex
\section{Methodology}
In this section, we propose the SL-SAC algorithm. Our approach addresses the challenges of safe exploration and robust constraint satisfaction by integrating three synergistic components: (1) ensemble-based reward estimation optimized via aSGLD, which encourages parameter diversity and helps escape sharp minima in the value landscape; (2) distributional cost assessment employing IQN to enforce CVaR-based safety limits; and (3) a reactive Lagrangian relaxation scheme that adapts constraint enforcement based on empirical risk realization. We detail each component and their integration below.

\begin{algorithm}[t]
\caption{Safe Langevin Soft Actor Critic (SL-SAC)}
\label{alg:abc}
\begin{algorithmic}[1]
\STATE \textbf{Input:} Cost limit $\beta$, ensemble size $M$, CVaR level $\epsilon$, steps per epoch $T_{\text{end}}$
\STATE \textbf{Initialize:} Reward critics $\{Q^{(i)}_\phi\}_{i=1}^{2M}$, cost critic $Z^C_\psi$, policy $\pi_\theta$, Multiplier $\lambda$, replay buffer $\mathcal{D}$, episode cost window $\mathcal{W}$, episode cost $C_{\text{ep}} \leftarrow 0$
\STATE \textbf{Initialize:} Target networks $\phi' \leftarrow \phi$, $\psi' \leftarrow \psi$, $\theta' \leftarrow \theta$
\FOR{epoch $= 1, 2, \ldots$}
    \FOR{$t = 1, 2, \ldots, T_{\text{end}}$}
        \STATE Sample action $a_t \sim \pi_\theta(\cdot|s_t)$, observe $r_t, c_t, s_{t+1}, d_t$
        \STATE Store $(s_t, a_t, r_t, c_t, s_{t+1}, d_t)$ in $\mathcal{D}$
        \STATE $C_{\text{ep}} \leftarrow C_{\text{ep}} + c_t$ 
        \IF{$d_t$} 
            \STATE Store $C_{\text{ep}}$ in $\mathcal{W}$, $C_{\text{ep}} \leftarrow 0$
        \ENDIF
        \STATE Sample minibatch $B$ from $\mathcal{D}$
        \STATE Update $\{Q^{(i)}_\phi\}$ using Algorithm~\ref{alg:asgld_update}
        \STATE Sample quantiles $\tau, \tau' \sim \mathcal{U}[0,1]$
        \STATE Update $Z^C_\psi$ via quantile regression (Eq.~\ref{eq:iqn_loss})
        \STATE Estimate $\widehat{\mathrm{CVaR}}_\epsilon(s,a)$ (Eq.~\ref{eq:cvar_estimate})
        \IF{$t \bmod 2 = 0$}
            \STATE Update $\pi_\theta$ via Eq.~\ref{eq:policy_objective}
            \STATE Soft update all target networks with factor $\tau_{\text{soft}}$
        \ENDIF
        \STATE Update multiplier $\lambda$ (Eq.~\ref{eq:lambda_update})
    \ENDFOR
\ENDFOR
\end{algorithmic}
\end{algorithm}

\subsection{Reward Value Estimation}

SL-SAC employs an ensemble of $M$ twin-critic pairs, yielding $2M$ total reward networks $\{Q^{(i)}_\phi\}_{i=1}^{2M}$. Each pair consists of two independently initialized critics $(Q^{(2m-1)}_\phi, Q^{(2m)}_\phi)$ for $m=1,\dots,M$. Diversity among ensemble members arises from both random initialization and the stochastic optimization dynamics of aSGLD.

\textbf{Ensemble Training Objective.} Each critic minimizes the Bellman error over transitions $(s, a, r, s', d)$ sampled from the replay buffer $\mathcal{D}$. The loss for the $i$-th critic is:
\begin{equation}
    \mathcal{L}^{\text{reward}}_i(\phi) = \mathbb{E}_{(s,a,r,s') \sim \mathcal{D}} \left[ \left(Q^{(i)}_\phi(s,a) - y^{\text{reward}}\right)^2 \right],
    \label{eq:reward_critic_loss}
\end{equation}
with target value:
\begin{equation}
\begin{split}
    y^{\text{reward}} = r &+ \gamma (1 - d) \Bigg( \frac{1}{M}\sum_{m=1}^M \min_{j \in \{2m-1, 2m\}} Q^{(j)}_{\phi'}(s', a') \\
    &- \alpha \log \pi_{\theta'}(a'|s') \Bigg),
    \label{eq:reward_target}
\end{split}
\end{equation}
where $a' \sim \pi_{\theta'}(\cdot|s')$, $\phi'$ and $\theta'$ are target parameters. 

\textbf{Optimization via aSGLD.} Standard gradient-based optimizers such as Adam tend to converge to sharp minima in the loss landscape, which can lead to overfitting and poor generalization in value function approximation~\citep{ishfaq2025langevin}. To promote convergence to flatter minima and enhance diversity across our reward critic ensemble, we employ aSGLD. Specifically, we utilize an update rule that combines Adam-style adaptive drift for fast convergence with isotropic Gaussian noise for robust exploration.

We update the critic parameters $\phi$ using the following rule:
\begin{equation}
\begin{split}
    \phi_{k+1} \leftarrow \phi_k &- \eta \left( \nabla_\phi \mathcal{L}^{\text{reward}}(\phi_k) + a \zeta_{\phi_k} \right) \\
    &+ \sqrt{2\eta T^{-1}} \, \xi_k, \quad \text{where } \xi_k \sim \mathcal{N}(0, I_d),
\end{split}
\label{eq:asgld_update}
\end{equation}
where $\eta$ is the learning rate, $T^{-1}$ is the inverse temperature parameter, and $\xi_k$ is standard isotropic Gaussian noise. The term $\zeta_{\phi_k}$ represents the adaptive preconditioner derived from the first and second gradient moments (similar to Adam), and $a$ is a bias factor. By injecting noise into the optimization trajectory, aSGLD facilitates the escape from sharp local minima. Concurrently, the independent stochastic updates of the aSGLD optimizer induce parameter diversity within the ensemble, counteracting the tendency of the critics to converge to identical weights despite the shared regression targets. The adaptive term $\zeta_{\phi_k}$ is defined as $\zeta_{\phi_k} = m_k \oslash \sqrt{v_k + \varepsilon I}$, where $m_k$ and $v_k$ are exponential moving averages of the gradient and its square, respectively. This process, which can be interpreted through a probabilistic lens (see Appendix~\ref{sec:probabilistic_critic}), encourages parameter diversity within the ensemble, leading to more robust value estimates.

\begin{algorithm}[b]
   \caption{Reward Update via aSGLD}
   \label{alg:asgld_update}
\begin{algorithmic}[1]
   \STATE {\bfseries Input:} Critic ensemble $\{Q_{\phi_i}\}_{i=1}^{2M}$, Batch $\mathcal{B}$.
   \STATE {\bfseries Hyperparameters:} Learning rate $\eta$, Bias factor $a$, Inverse temperature $T^{-1}$.
   
   \STATE Compute Bellman targets $y$ for batch $\mathcal{B}$ using target ensemble (Eq.~\ref{eq:reward_target}).
   
   \FOR{each critic $i \in \{1, \dots, 2M\}$}
       \STATE Compute gradient $g_i = \nabla_{\phi_i} \mathcal{L}^{\text{reward}}(\phi_i; \mathcal{B})$ (Eq.~\ref{eq:reward_critic_loss}).
       
       \STATE Update optimizer moments $m_i, v_i$ using gradients $g_i$.
       \STATE Compute adaptive preconditioner $\zeta_i = m_i \oslash (\sqrt{v_i + \varepsilon I})$.
       
       \STATE Sample isotropic noise $\xi_i \sim \mathcal{N}(0, I)$.
       
       \STATE $\phi_i \leftarrow \phi_i - \eta (g_i + a \cdot \zeta_i) + \sqrt{2\eta T^{-1}} \, \xi_i$
   \ENDFOR
\end{algorithmic}
\end{algorithm}

\begin{figure*}[t]
    \centering
    \includegraphics[width=\linewidth]{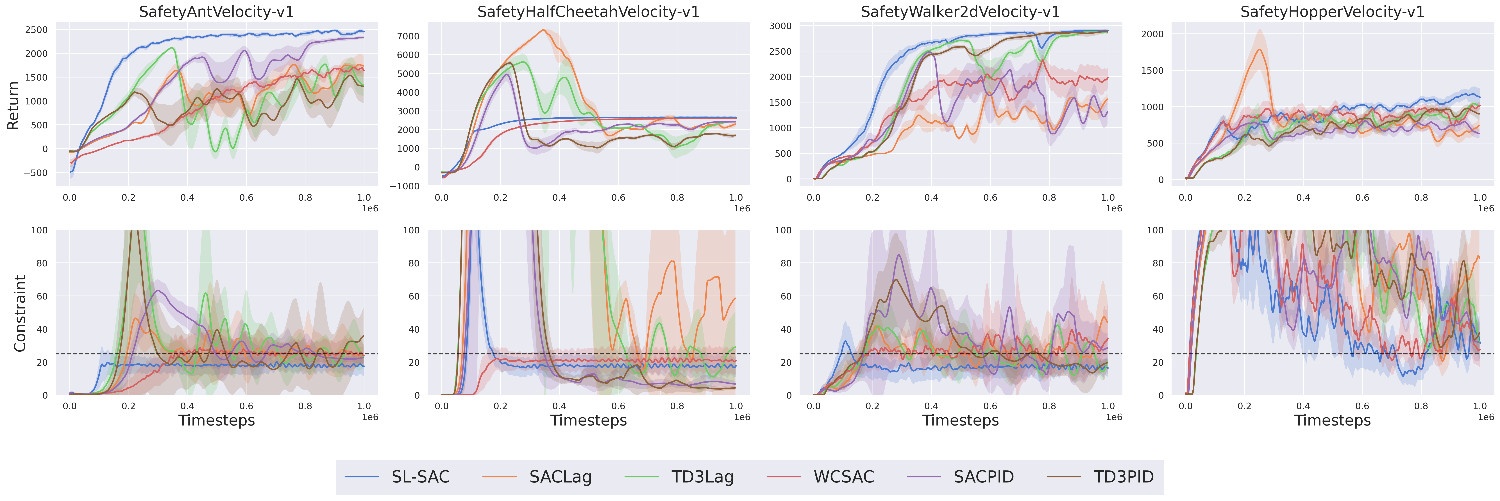}
    \caption{Training curves for safety-mujoco velocity tasks. Episode return (left) and cost (right), averaged over 5 seeds (solid line: mean, shaded area: standard deviation). The dashed line indicates the constraint threshold $\beta=25$.}
    \label{fig:main_velocity_results}
\end{figure*}

\subsection{Risk Assessment}

Ensuring reliable constraint satisfaction in stochastic environments necessitates the quantification of tail risks rather than expected costs. To this end, we employ a Distributional Critic based on IQN to model the full distribution of cost returns $Z^C(s,a)$. Unlike the reward ensemble, we use a single cost critic optimized via AdamW~\citep{loshchilov2017decoupled} to maintain  quantile estimates. As demonstrated in our ablation study (Appendix~\ref{app:cost_optimizer_ablation}), applying aSGLD to the distributional cost critic destabilizes the quantile regression process, leading to inaccurate CVaR estimates and constraint violations. AdamW provides the necessary stability for precise tail-risk modeling.

The cost critic $Z^C_\psi$ learns the quantile function $F^{-1}_{Z^C}(\tau)$ by minimizing the quantile Huber loss. For a transition $(s,a,c,s',d)$ and independently sampled quantiles $\tau, \tau' \sim \mathcal{U}[0,1]$, the temporal difference error is:
\begin{equation}
    \delta_{\tau,\tau'} = c + \gamma_C (1-d) Z^{C,\tau'}_{\psi'}(s', a') - Z^{C,\tau}_\psi(s,a),
\end{equation}
The quantile Huber loss is:
\begin{equation}
    \rho_\kappa^\tau(\delta) = |\tau - \mathbb{I}_{\{\delta < 0\}}| \cdot \mathcal{L}_\kappa(\delta),
\end{equation}
with Huber loss $\mathcal{L}_\kappa(\delta)$ having threshold $\kappa$:
\begin{equation}
    \mathcal{L}_\kappa(\delta) = \begin{cases} 
        \frac{1}{2} \delta^2, & \text{if } |\delta| \le \kappa \\
        \kappa (|\delta| - \frac{1}{2} \kappa), & \text{otherwise}.
    \end{cases}
\end{equation}
The overall cost critic loss is:
\begin{equation}
    \mathcal{L}^{\text{cost}}(\psi) = \mathbb{E}_{(s,a,c,s',d) \sim \mathcal{D}, \tau, \tau' \sim \mathcal{U}[0,1]} \left[ \rho_\kappa^\tau \left( \delta_{\tau,\tau'} \right) \right].
    \label{eq:iqn_loss}
\end{equation}

Using the learned distribution, we estimate CVaR at level $\epsilon$ by averaging $N$ quantiles from the upper tail:
\begin{equation}
    \widehat{\text{CVaR}}_\epsilon(s,a) \approx \frac{1}{N} \sum_{k=1}^N Z^{C,\tau_k}_\psi(s,a), \quad \tau_k \sim \mathcal{U}[1-\epsilon, 1].
    \label{eq:cvar_estimate}
\end{equation}

\subsection{Policy Optimization}

The policy $\pi_\theta$ is trained to maximize expected reward and entropy while satisfying a CVaR-based safety constraint, formulated as a Lagrangian optimization problem (see Algorithm~\ref{alg:abc} and Figure~\ref{fig:abc_arch}):
\begin{equation}
    \begin{aligned}
    \max_\theta \, \mathbb{E}_{s \sim \mathcal{D}, a \sim \pi_\theta(\cdot|s)} \Big[ & \bar{Q}_{\phi}(s,a) - \alpha \log \pi_\theta(a|s) \\
    & - \lambda \cdot \widehat{\text{CVaR}}_\epsilon[Z^C_{\psi}(s,a)] \Big],
    \end{aligned}
    \label{eq:policy_objective}
\end{equation}
where $\bar{Q}_{\phi}(s, a)$ is the average across the $M$ twin-critic pairs of the minimum Q-value, and $\widehat{\text{CVaR}}_\epsilon$ is estimated via Eq.~\ref{eq:cvar_estimate}. 

The multiplier $\lambda$ is adapted based on the agent's empirical safety performance. We maintain a sliding window $\mathcal{W}$ of the most recent episodic costs. The multiplier is then updated via projected gradient ascent:
\begin{equation}
    \lambda \leftarrow \max\left(0, \lambda + \eta_\lambda \left( \text{CVaR}^{\text{empirical}}_\epsilon[\mathcal{W}] - \beta \right) \right),
    \label{eq:lambda_update}
\end{equation}

We update $\lambda$ using empirical CVaR rather than the critic's estimate to ensure the penalty reflects realized safety performance, independent of value estimation errors.

\begin{table*}[t]
    \centering
    \caption{\textbf{Performance comparison on Safety-Gymnasium tasks.} Episode return ($r$, higher $\uparrow$) and episode cost ($c$, lower $\downarrow$). \textcolor{unsafegray}{Gray} indicates constraint violation ($>25$), Black indicates safety, and \best{Green} indicates the highest return satisfying safety constraints.}
    \label{tab:main_results}
    
    \resizebox{\textwidth}{!}{%
    \begin{tabular}{ll cccccc}
        \toprule
        \textbf{Task} & \textbf{Metric} & \textbf{SACLag} & \textbf{TD3Lag} & \textbf{SACPID} & \textbf{TD3PID} & \textbf{WCSAC} & \textbf{SL-SAC} \\
        \midrule
        
        % --- Environment 1: Ant ---
        % Analysis: TD3PID(40.47) is unsafe. SL-SAC(3104) is highest safe return.
        \multirow{2}{*}{SafetyAntVelocity-v1} 
        & $r \uparrow$ & $1706.76 \pm 135.86$ & $1613.25 \pm 444.47$ & $1822.92 \pm 81.58$ & \bad{$2947.76 \pm 535.29$} & $2846.36 \pm 13.01$ & \best{3104.49 $\pm$ 19.60} \\
        & $c \downarrow$ & $13.68 \pm 5.84$ & $23.86 \pm 4.92$ & $4.58 \pm 1.91$ & \bad{$40.47 \pm 9.58$} & $10.60 \pm 2.92$ & \best{3.57 $\pm$ 0.45} \\
        \midrule
        
        % --- Environment 2: HalfCheetah ---
        % Analysis: TD3PID(44.62) unsafe. SL-SAC(2828) highest safe.
        \multirow{2}{*}{SafetyHalfCheetahVelocity-v1} 
        & $r \uparrow$ & $2776.93 \pm 3.15$ & $2706.67 \pm 77.12$ & $2645.85 \pm 94.44$ & \bad{$2739.13 \pm 141.38$} & $2814.19 \pm 22.06$ & \best{2828.30 $\pm$ 13.78} \\
        & $c \downarrow$ & $07.17 \pm 9.25$ & $13.78 \pm 14.70$ & $12.41 \pm 1.59$ & \bad{$44.62 \pm 0.88$} & $0.00 \pm 0.00$ & \best{0.00 $\pm$ 0.00} \\
        \midrule
        
        % --- Environment 3: Walker2d ---
        % Analysis: SACLag(28.21), SACPID(25.50), WCSAC(29.10) are UNSAFE. 
        % Safe: TD3Lag(2873), TD3PID(2861), SL-SAC(3003). Winner: SL-SAC.
        \multirow{2}{*}{SafetyWalker2dVelocity-v1} 
        & $r \uparrow$ & \bad{$2582.17 \pm 113.49$} & $2873.17 \pm 76.04$ & \bad{$1871.99 \pm 570.99$} & $2861.36 \pm 158.88$ & \bad{$2705.12 \pm 175.62$} & \best{3003.84 $\pm$ 12.30} \\
        & $c \downarrow$ & \bad{$28.21 \pm 7.81$} & $23.56 \pm 1.64$ & \bad{$25.50 \pm 2.02$} & $13.89 \pm 2.04$ & \bad{$29.10 \pm 29.11$} & \best{5.10 $\pm$ 6.72} \\
        \midrule
        
        % --- Environment 4: Swimmer ---
        % Analysis: SACLag(26.92) and TD3PID(25.54) are UNSAFE. 
        % Safe: TD3Lag(21.24), SACPID(7.99), WCSAC(18.19), SL-SAC(13.06).
        % Winner: TD3Lag is the highest SAFE return.
        \multirow{2}{*}{SafetySwimmerVelocity-v1} 
        & $r \uparrow$ & \bad{$14.10 \pm 2.75$} & \best{21.24 $\pm$ 3.71} & $7.99 \pm 1.15$ & \bad{$43.14 \pm 3.67$} & $18.19 \pm 14.90$ & $13.06 \pm 12.37$ \\
        & $c \downarrow$ & \bad{$26.92 \pm 8.95$} & \best{15.42 $\pm$ 2.82} & $5.78 \pm 1.07$ & \bad{$25.54 \pm 0.65$} & $16.27 \pm 16.28$ & $2.40 \pm 7.00$ \\
        \midrule
        
        % --- Environment 5: Humanoid ---
        % Analysis: All safe. SACPID(5428) is highest.
        \multirow{2}{*}{SafetyHumanoidVelocity-v1} 
        & $r \uparrow$ & $5141.52 \pm 206.28$ & $5405.15 \pm 116.32$ & \best{5428.33 $\pm$ 25.92} & $5266.17 \pm 227.36$ & $2995.18 \pm 412.86$ & $5183.66 \pm 109.39$ \\
        & $c \downarrow$ & $0.25 \pm 0.16$ & $0.23 \pm 0.23$ & \best{0.16 $\pm$ 0.11} & $0.30 \pm 0.39$ & $7.20 \pm 3.34$ & $0.00 \pm 0.00$ \\
        \midrule
        
        % --- Environment 6: Hopper ---
        % Analysis: TD3Lag(36.1), SACPID(48.5), TD3PID(44.1) are UNSAFE.
        % Safe: SACLag(935), WCSAC(1040), SL-SAC(1292). Winner: SL-SAC.
        \multirow{2}{*}{SafetyHopperVelocity-v1} 
        & $r \uparrow$ & $935.42 \pm 38.40$ & \bad{$990.79 \pm 389.48$} & \bad{$997.47 \pm 25.14$} & \bad{$1120.09 \pm 42.64$} & $1040.07 \pm 5.12$ & \best{1292.61 $\pm$ 107.73} \\
        & $c \downarrow$ & $18.56 \pm 18.20$ & \bad{$36.16 \pm 20.33$} & \bad{$48.53 \pm 9.48$} & \bad{$44.17 \pm 68.23$} & $15.40 \pm 38.18$ & \best{12.53 $\pm$ 23.38} \\
        \midrule
        
        % --- CarCircle ---
        % Analysis: SACLag(29.29) UNSAFE. 
        % Safe: TD3Lag(26.69), SACPID(28.39), TD3PID(17.10), WCSAC(16.97), SL-SAC(17.67). Winner: SACPID.
        \multirow{2}{*}{SafetyCarCircle1-v0} 
        & $r \uparrow$ & \bad{$19.60 \pm 10.35$} & $26.69 \pm 0.67$ & \best{28.39 $\pm$ 10.79} & $17.10 \pm 0.65$ & $16.97 \pm 0.76$ & $17.67 \pm 0.40$ \\
        & $c \downarrow$ & \bad{$29.29 \pm 11.79$} & $15.05 \pm 22.20$ & \best{11.27 $\pm$ 3.25} & $24.64 \pm 1.20$ & $15.80 \pm 22.34$ & $8.87 \pm 12.54$ \\
        \midrule
        
        % --- CarGoal ---
        % Analysis: TD3PID(56.21) UNSAFE.
        % Safe: SACLag(11.88), TD3Lag(21.61), SACPID(23.94), WCSAC(10.34), SL-SAC(22.16). Winner: SACPID.
        \multirow{2}{*}{SafetyCarGoal1-v0} 
        & $r \uparrow$ & $11.88 \pm 1.53$ & $21.61 \pm 3.57$ & \best{23.94 $\pm$ 1.00} & \bad{$32.21 \pm 2.46$} & $10.34 \pm 2.17$ & $22.16 \pm 4.29$ \\
        & $c \downarrow$ & $17.31 \pm 7.09$ & $15.75 \pm 3.04$ & \best{21.52 $\pm$ 4.51} & \bad{$56.21 \pm 2.01$} & $6.80 \pm 4.43$ & $10.07 \pm 2.33$ \\
        \midrule
        
        % --- PointCircle ---
        % Analysis: SACLag(28.99) UNSAFE.
        % Safe: TD3Lag(69.57), SACPID(62.10), TD3PID(51.16), etc. Winner: TD3Lag.
        \multirow{2}{*}{SafetyPointCircle1-v0} 
        & $r \uparrow$ & \bad{$71.61 \pm 5.46$} & \best{69.57 $\pm$ 1.33} & $62.10 \pm 0.53$ & $51.16 \pm 1.29$ & $36.38 \pm 3.36$ & $46.74 \pm 3.64$ \\
        & $c \downarrow$ & \bad{$28.99 \pm 1.77$} & \best{04.02 $\pm$ 7.75} & $12.79 \pm 1.92$ & $03.91 \pm 4.54$ & $10.70 \pm 15.13$ & $19.17 \pm 20.16$ \\
        \midrule
        
        % --- PointGoal ---
        % Analysis: TD3Lag(29.20), SACPID(36.42) UNSAFE.
        % Safe: SACLag(15.04), TD3PID(16.35), WCSAC(5.48), SL-SAC(11.04). Winner: TD3PID.
        \multirow{2}{*}{SafetyPointGoal1-v0} 
        & $r \uparrow$ & $15.04 \pm 2.71$ & \bad{$22.83 \pm 3.29$} & \bad{$14.90 \pm 5.21$} & \best{16.35 $\pm$ 7.23} & $5.48 \pm 2.00$ & $11.04 \pm 0.72$ \\
        & $c \downarrow$ & $19.98 \pm 7.93$ & \bad{$29.20 \pm 13.58$} & \bad{$36.42 \pm 2.53$} & \best{17.40 $\pm$ 9.62} & $3.17 \pm 1.92$ & $16.93 \pm 2.54$ \\    
        
        \bottomrule
    \end{tabular}
    }
\end{table*}

\subsection{Sensitivity Analysis of CVaR Estimates}

While Lemma \ref{lemma:contraction} ensures convergence to the true distribution $Z^\pi$, practical safety depends on the accuracy of the risk metric derived from the approximated distribution $Z_\psi$. The following theorem provides a bound on CVaR estimation error in terms of the mean squared quantile error.

\begin{theorem}[CVaR Error Bound via Quantile Estimation Error]
\label{thm:cvar_error}
Let $Z^\pi(s,a)$ be the true cost return distribution and $Z_\psi(s,a)$ be the IQN approximation with parameters $\psi$. Define the mean squared quantile error:
\begin{equation}
\delta^2 := \mathbb{E}_{\tau \sim \mathrm{Unif}[0,1]} \left[ \left| F^{-1}_{Z^\pi(s,a)}(\tau) - Z_\psi(s,a; \tau) \right|^2 \right],
\end{equation}
where $Z_\psi(s,a; \tau)$ denotes the estimated $\tau$-quantile. Then for any risk level $\epsilon \in (0,1)$:
\begin{equation}
\left| \mathrm{CVaR}_\epsilon(Z^\pi(s,a)) - \mathrm{CVaR}_\epsilon(Z_\psi(s,a)) \right| \leq \frac{\delta}{\sqrt{\epsilon}}.
\end{equation}
\end{theorem}

Theorem \ref{thm:cvar_error} bounds the CVaR estimation error by the root mean squared quantile error scaled by $1/\sqrt{\epsilon}$. The proof is provided in Appendix~\ref{thm:cvar_error_appendix}.

\begin{corollary}[Tail Risk Guarantees: CVaR vs. Expected Cost Constraints]
\label{cor:cvar_vs_expected}
Consider a policy $\pi$ with non-negative cost distribution $Z^\pi$ and IQN approximation $Z_\psi$ achieving mean squared quantile error $\delta^2$ as in Theorem \ref{thm:cvar_error}. For a given cost limit $\beta$ and risk level $\epsilon \in (0,1)$, the following guarantees hold:

\begin{itemize}
    \item \textbf{CVaR-based constraint:} If $\widehat{\mathrm{CVaR}}_\epsilon(Z_\psi) \leq \beta$, then
    \begin{equation}
        \mathrm{CVaR}_\epsilon(Z^\pi) \leq \beta + \frac{\delta}{\sqrt{\epsilon}}.
        \label{eq:cvar_constraint_bound}
    \end{equation}
    
    \item \textbf{Expected cost constraint:} If $\mathbb{E}[Z_\psi] \leq \beta$, then by the worst‑case tail distribution:
    \begin{equation}
        \mathrm{CVaR}_\epsilon(Z^\pi) \leq \frac{\beta + \delta}{\epsilon}.
        \label{eq:expected_constraint_bound}
    \end{equation}
\end{itemize}

The CVaR‑based constraint yields a strictly tighter bound on true tail risk whenever
\begin{equation}
    \beta + \frac{\delta}{\sqrt{\epsilon}} < \frac{\beta + \delta}{\epsilon},
\end{equation}
which holds for all $\beta > 0$, $\delta > 0$, and $\epsilon \in (0,1)$.
\end{corollary}

Corollary~\ref{cor:cvar_vs_expected} demonstrates that direct CVaR estimation, as employed in SL‑SAC, provides stronger theoretical guarantees on tail risk compared to constraints based on expected cost. The proof is provided in Appendix~\ref{cor:cvar_vs_expected_appendix}.

\begin{theorem}[Constraint Violation Probability under CVaR Constraint]
\label{thm:cvar_gpd_violation}
Assuming the tail of the cost distribution follows a Generalized Pareto Distribution (GPD) with shape parameter $\nu < 1$, satisfying the CVaR constraint $\mathrm{CVaR}_\epsilon(Z^\pi) \leq \beta$ provides an upper bound on the violation probability:
\[
\Pr(Z^\pi > \beta) \leq (1-\epsilon) \cdot \gamma(\nu),
\]
where $\gamma(\nu) = e^{-1}$ for exponential tails ($\nu=0$) and $\gamma(\nu) = (1-\nu)^{1/\nu}$ for heavy tails ($0<\nu<1$). The proof is provided in Appendix \ref{app:cvar_gpd_violation_appendix}.

\end{theorem}

%% file: experiments.tex
\section{Experiments}
\label{sec:experiments}

\begin{figure*}[ht!]
    \centering
    \includegraphics[width=\linewidth]{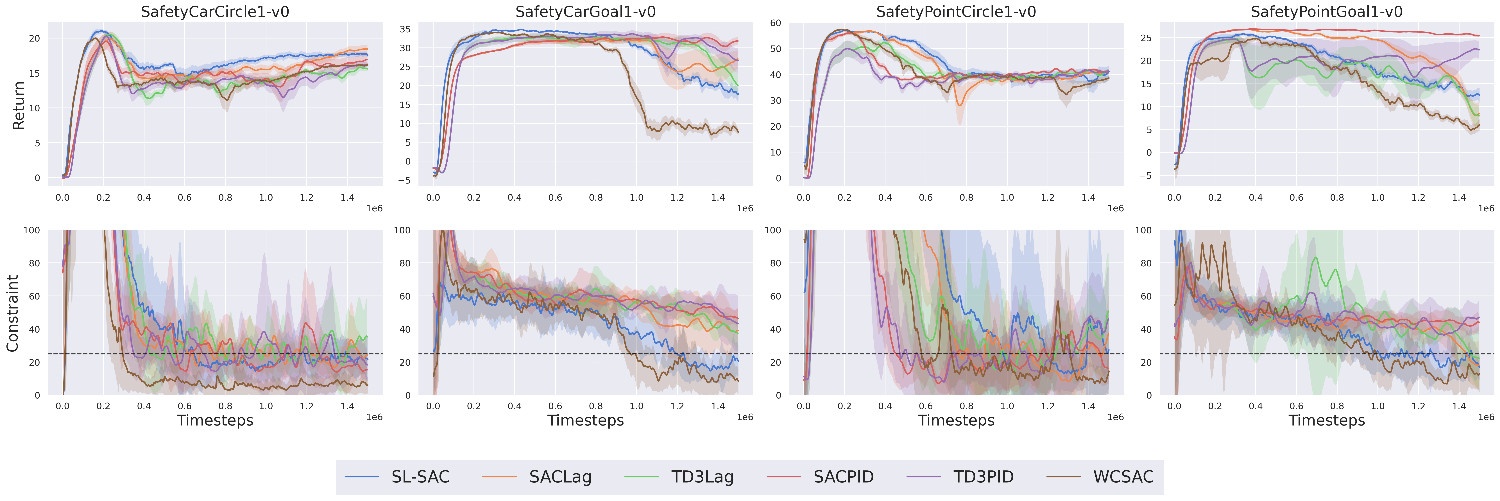}
    \caption{Training curves for safety-navigation tasks. Episode return (left) and cost (right), averaged over 5 seeds (solid line: mean, shaded area: standard deviation). The dashed line indicates the constraint threshold $\beta=25$.}
    \label{fig:main_navigation_results}
\end{figure*}

We evaluate SL-SAC on the Safety-Gymnasium benchmark~\citep{ji2023safety}. Each environment provides a reward signal $r(s,a)$ that encourages task completion and a cost signal $c(s, a)$ that penalizes constraint violations. Costs are typically incurred when the agent enters unsafe regions, collides with obstacles, or exceeds velocity limits. The safety threshold is set to $\beta=25$ for all tasks. Implementation details and environment descriptions are provided in Appendix~\ref{sec:experimental_setup}. For each algorithm, we perform training over 5 random seeds and evaluate performance using 30 deterministic evaluation rollouts. We report the mean episodic return and cost with standard errors across seeds.

We compare SL-SAC with state-of-the-art off-policy SafeRL algorithms, including Lagrangian variants of SAC (SAC-Lag) and TD3 (TD3-Lag)~\citep{ray2019benchmarking,fujimoto2018addressing}, their PID-controlled counterparts (SAC-PID, TD3-PID)~\citep{stooke2020responsive}, and WCSAC with non-parametric cost estimation~\citep{yang2023safety}. For completeness, we also include comparisons with on-policy methods, CPO~\citep{achiam2017constrained} and PPOLag~\citep{ray2019benchmarking} in Appendix~\ref{subsec:on_policy_comparison}, though note that their sample efficiency and performance characteristics differ due to the distinct on-policy training paradigm.

Table~\ref{tab:main_results} summarizes the evaluation results, while Figures~\ref{fig:main_velocity_results} and~\ref{fig:main_navigation_results} depict the training dynamics. Across the diverse set of Velocity and Navigation tasks, SL-SAC demonstrates consistent constraint satisfaction, maintaining average episodic costs below the safety threshold of 25 in nearly all environments. In terms of task performance, the algorithm achieves cumulative returns that are competitive with, and in some cases superior to, state-of-the-art baselines. To control for model capacity, we provide an ablation in Appendix~\ref{subsec:m1_capacity_ablation} using a single critic pair $M = 1$ for SL-SAC. The results indicate that SL-SAC remains effective even with same parameters count, suggesting that our improvements are not driven solely by the larger ensemble size.

\paragraph{Evaluation on Autonomous Driving.}
To validate performance in a domain with dynamic constraints and high stochasticity, we evaluate SL-SAC on the MetaDrive simulator~\citep{li2022metadrive}. We adopt the specific environment configuration and task definition from \citet{zhang2024safe}, which serves as a robust testbed for generalizing safety constraints under uncertainty. This setting introduces significant epistemic uncertainty due to varying road topologies and aleatoric uncertainty from stochastic traffic behavior.

The environment consists of three-block road sequences stochastically generated from a set of primitives (Straight, Curve, Intersection, Roundabout) with sampling probabilities $(0.2, 0.3, 0.3, 0.2)$. The agent controls a vehicle using Lidar-based observations and ego-state information, navigating through moderate traffic density $0.12$. A unit cost is incurred for collisions with vehicles, obstacles, or driving off-road. The constraint threshold is set to $\beta = 0.1$, a strict safety requirement aiming for near-zero collisions.

Figure~\ref{fig:main_metadrive_results} depicts the training dynamics. SL-SAC and WCSAC emerge as the top-performing algorithms, significantly outperforming PID-based and standard Lagrangian baselines. We report the cumulative episodic cost, as shown in the results, strictly satisfying the tight threshold ($\beta=0.1$) is challenging for all methods due to the unavoidable aleatoric uncertainty in traffic generation, resulting in steady-state costs above the threshold. However, SL-SAC stabilizes at a lower average cost than baselines while maintaining high returns.

\begin{figure}[hb!]
    \centering
    \includegraphics[width=\linewidth]{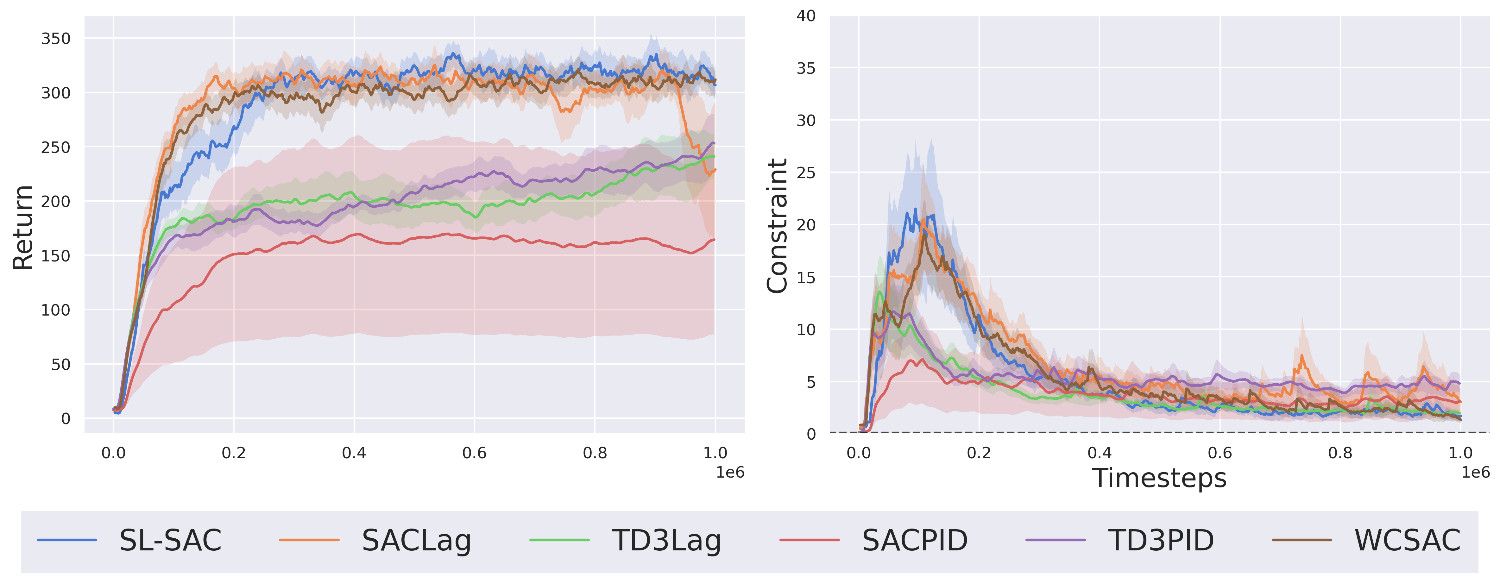}
    \caption{Training curves for MetaDrive. Episode return (left) and cost (right), averaged over 5 seeds (solid line: mean, shaded area: standard deviation). The dashed line indicates the constraint threshold $\beta=0.1$.}
    \label{fig:main_metadrive_results}
\end{figure}
\paragraph{Impact of optimizer.}
We compare aSGLD optimizer against AdamW for training the ensemble of reward critics, while keeping all other components fixed. As shown in Figure~\ref{fig:ant_adam_asgld}, aSGLD consistently yields higher asymptotic returns and lower cost variance. When AdamW is used, the ensemble reduces to a standard deep ensemble trained via independent initialization. In contrast, the aSGLD variant, which combines adaptive drift with isotropic noise injection, actively promotes parameter diversity and mitigates ensemble collapse. This leads to more robust value estimation and improved exploration, which translates into the performance gains observed in our experiments.

% \paragraph{CVaR vs. Expected Value for Constraint Enforcement}

% We compare two approaches for updating the Lagrange multiplier $\lambda$: using CVaR$_\epsilon$ of recent episodic costs versus using their expected value. As shown in Figure~\ref{fig:adam_vs_asgld_expected_cvar}, CVaR-based updates lead to more stable learning compared to expected value. When using $\mathbb{E}[\mathcal{W}]$, particularly in Walker, the agent exhibits increased variance in both reward and cost. In contrast, CVaR maintains consistent performance with costs remaining below the safety threshold. This difference can be attributed to CVaR's focus on tail outcomes, the expected value aggregates over all episodes equally, while CVaR emphasizes worst-case scenarios, leading to more conservative multiplier adjustments that improve training stability.

\begin{figure*}[ht!] % The * spans both columns. [t] tries to place it at top of page.
    \centering
    
    % --- Row 1 ---
    \begin{subfigure}[b]{0.48\textwidth}
        \centering
        % Replace with your first environment image
        \includegraphics[width=\linewidth]{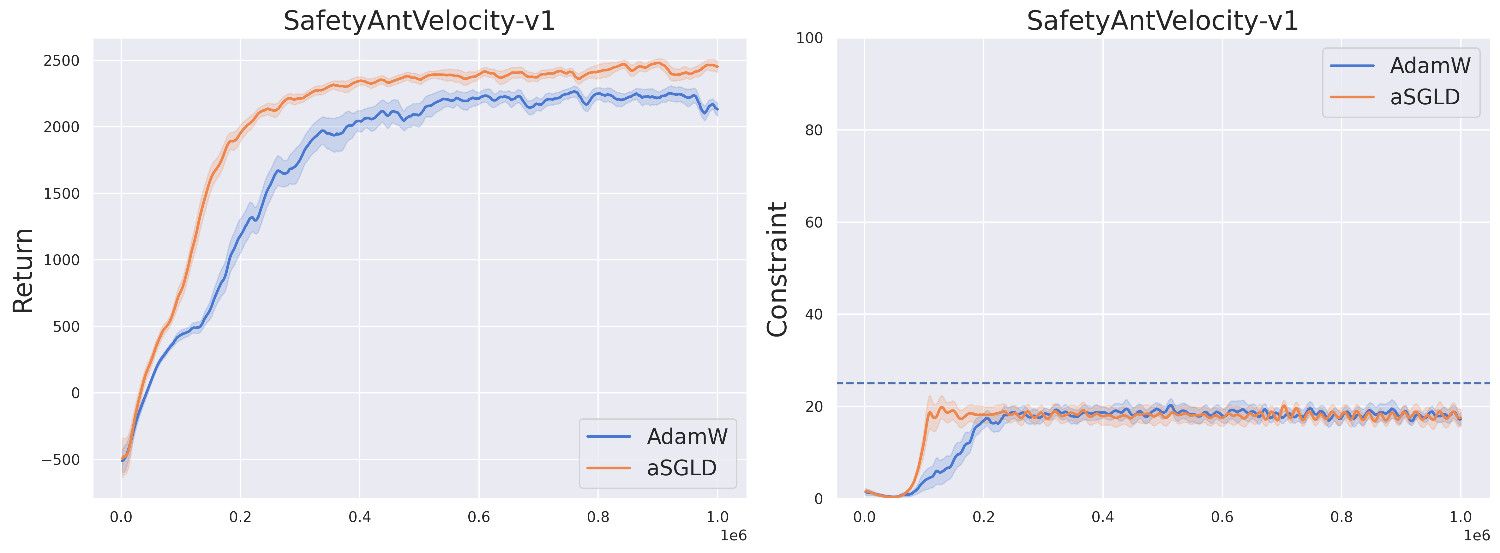} 
        \caption{aSGLD vs AdamW}
        \label{fig:ant_adam_asgld}
    \end{subfigure}
    \hfill % Adds flexible space between the two images
    \begin{subfigure}[b]{0.48\textwidth}
        \centering
        % Replace with your second environment image
        \includegraphics[width=\linewidth]{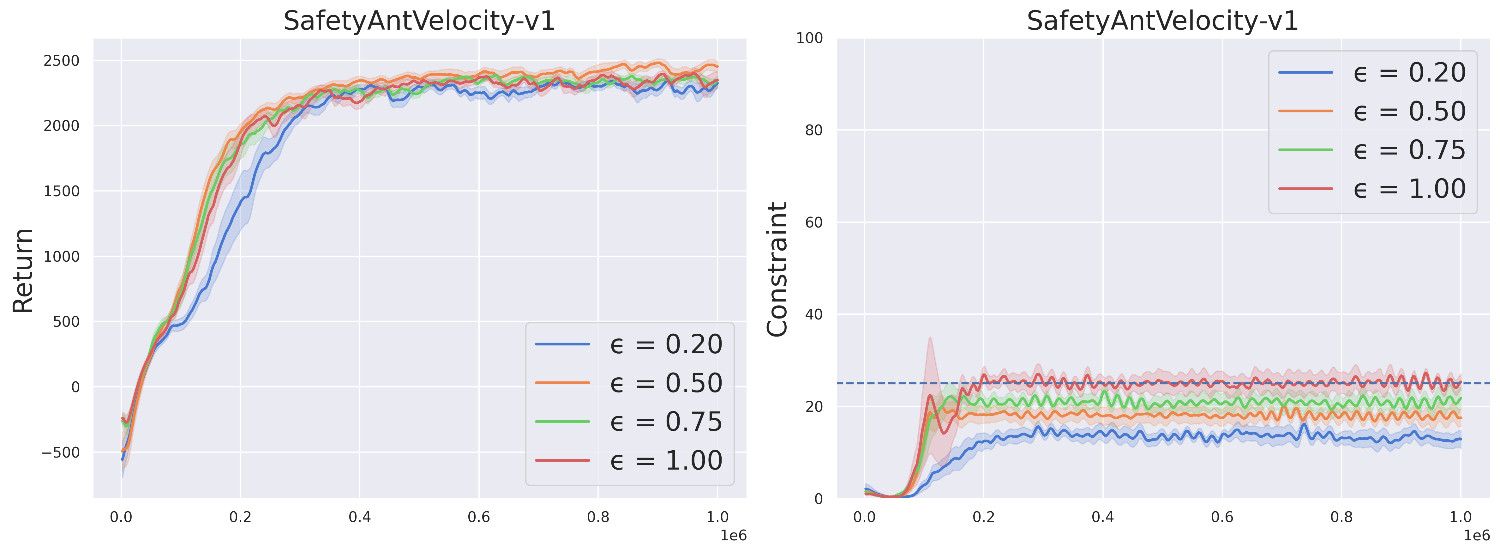}
        \caption{CVaR ablation ($\epsilon$)}
        \label{fig:ant_cvar_ablation}
    \end{subfigure}
    
    \caption{\textbf{Ablation studies: optimizer and CVaR level.} Left: Reward critic optimizer (aSGLD vs AdamW). Right: CVaR confidence level $\epsilon \in \{0.2, 0.5, 0.75, 1.0\}$ ($\epsilon=1.0$ is expected cost).}
    \label{fig:ablations}
\end{figure*}

\paragraph{Sensitivity to CVaR Confidence Level ($\epsilon$).}
We investigate the sensitivity of our proposed method to the CVaR confidence level, $\epsilon$. Specifically, setting $\epsilon = 1.0$ is equivalent to optimizing the expected value of the cost distribution, while lower values of $\epsilon$ focus on the upper tail to mitigate worst-case outcomes. 

Figure~\ref{fig:ant_cvar_ablation} demonstrate a distinct relationship between the confidence level and the agent's safety profile. As observed in the cost plot, decreasing $\epsilon$ reduces the average cost incurred during training. Lower $\epsilon$ values induce safer behaviors, while the risk-neutral setting ($\epsilon=1.0$) frequently oscillate around or exceed the safety threshold.

Crucially, this improvement in safety does not come at the expense of task performance. The reward curves show that the agent achieves competitive returns across the spectrum of $\epsilon$ values, suggesting that the method effectively navigates the constrained optimization landscape without becoming overly conservative. 

%% file: related_work.tex
\section{Related Work}

SafeRL is predominantly formulated as Constrained Markov Decision Processes (CMDPs), typically solved via primal or primal-dual methods. Primal approaches, such as Constrained Policy Optimization (CPO)~\citep{achiam2017constrained} and PCPO~\citep{Yang2020ProjectionBased}, enforce constraints by projecting policy updates onto the feasible region. Conversely, primal-dual methods employ Lagrangian relaxation, converting the constrained problem into an unconstrained dual formulation~\citep{stooke2020responsive}. While theoretically grounded, these methods optimize for expected returns subject to expected costs, a risk-neutral formulation that can be inadequate in stochastic environments where heavy-tailed cost distributions may lead to catastrophic violations.

To address tail risks, recent works have incorporated risk measures like CVaR into SafeRL. WCSAC~\citep{yang2023safety} employs IQN to model the distributional safety critic, allowing for risk-sensitive policy optimization. However, WCSAC updates its Lagrange multiplier based on the critic's estimated risk. This creates a  dependency on approximation accuracy, if the distributional critic underestimates the tail risk early in training, the safety penalty remains insufficient, leading to constraint violations. In contrast, SL-SAC decouples the constraint enforcement by updating the multiplier based on the empirical CVaR of realized episodic costs. This ensures that safety penalties are grounded in actual agent performance rather than potentially erroneous value estimates. Other approaches incorporate CVaR constraints via policy gradient methods~\citep{chow2018risk} or optimize for varying risk levels through distributional modeling~\citep{tang2019worst}. However, these methods typically suffer from high gradient variance due to trajectory sampling or rely on restrictive parametric assumptions that fail to capture heavy-tailed cost distributions efficiently.

SGLD and its adaptive variants (aSGLD)~\citep{li2016preconditioned} inject controlled noise into gradient updates to promote parameter-space exploration. In RL, this has been leveraged to improve exploration and prevent mode collapse, as demonstrated in LSAC~\citep{ishfaq2025langevin}, which applies aSGLD to ensembles of reward critics. Regarding safety, \citet{lei2024langevin} recently proposed using Langevin dynamics to model the policy distribution for constraint satisfaction. While LSAC utilizes aSGLD for unconstrained reward maximization and \citet{lei2024langevin} focus on policy-based safety, SL-SAC extends the application of aSGLD-optimized critic ensembles to the Constrained MDP setting. By integrating robust reward estimation with distributional cost control, SL-SAC aims to improve value function robustness in safety-critical environments, where standard optimizers may converge to suboptimal safe policies.

%% file: conclusion.tex
\section{Conclusion}

This paper introduced Safe Langevin Soft Actor-Critic (SL-SAC), a risk-sensitive SafeRL algorithm that addresses poor generalization and tail-risk underestimation in constrained continuous control. SL-SAC combines three key mechanisms: aSGLD-optimized reward critics for improved value estimation, IQN-based distributional cost modeling for CVaR constraint enforcement, and CVaR-driven Lagrange updates for responsive safety adaptation. Theoretically, we provide guarantees on CVaR estimation error and prove that CVaR-based Lagrange updates yield stronger safety signals than expected-cost approaches. Empirically, SL-SAC achieves state-of-the-art safety-performance trade-offs, obtaining the lowest cost in 7 of 10 Safety-Gymnasium tasks while maintaining competitive returns. The improved safety and performance come with a modest computational overhead, as analyzed in Section~\ref{subsec:computational_efficiency}, due to the ensemble of reward critics and aSGLD optimizer. Future work focuses on mitigating this overhead by exploring more efficient Langevin optimizers and extending the algorithm to other risk measures.

%% file: appendix.tex
\section{Probabilistic Interpretation of Critic Updates}
\label{sec:probabilistic_critic}

To motivate the use of Adaptive SGLD for the ensemble of reward critics, we provide a probabilistic interpretation of the objective function. Let $\mathcal{D}_{\text{batch}} = \{(s_i, a_i, r_i, s'_i, d_i)\}_{i=1}^B$ be a minibatch of transitions sampled from the replay buffer, where $B$ is the batch size and $d_i$ indicates episode termination. The standard objective for training a critic $Q_\phi$ parameterized by $\phi$ is to minimize the Mean Squared Error (MSE) against the Bellman target $y_i$:

\begin{equation}
    \mathcal{L}_{\text{MSE}}(\phi) = \frac{1}{B} \sum_{i=1}^B \left( y_i - Q_\phi(s_i, a_i) \right)^2,
    \label{eq:mse_loss_appendix}
\end{equation}

where $y_i = r_i + \gamma (1-d_i) [Q_{\text{target}}(s'_i, a'_i) - \alpha \log \pi(a'_i|s'_i)]$. We model the Bellman targets as noisy observations of the critic's underlying predictions, assuming a Gaussian likelihood $y_i \sim \mathcal{N}(Q_\phi(s_i, a_i), \sigma^2)$. Here, $\sigma^2$ represents the assumed variance of the Bellman error, which captures aleatoric uncertainty in the target values. This yields the likelihood function for the parameters $\phi$:

\begin{equation}
    p(\mathcal{D}_{\text{batch}} | \phi) \propto \exp \left( - \frac{1}{2\sigma^2} \sum_{i=1}^B (y_i - Q_\phi(s_i, a_i))^2 \right).
\end{equation}

Furthermore, we impose a Gaussian prior on the network parameters $\phi \sim \mathcal{N}(0, \nu^{-1} I)$, where $\nu$ represents the precision of the prior distribution. By applying Bayes' rule, the posterior distribution over the critic parameters is:

\begin{equation}
    p(\phi | \mathcal{D}_{\text{batch}}) \propto p(\mathcal{D}_{\text{batch}} | \phi) p(\phi) \propto \exp \left( - \frac{1}{2\sigma^2} \sum_{i=1}^B (y_i - Q_\phi(s_i, a_i))^2 - \frac{\nu}{2} \|\phi\|^2 \right).
    \label{eq:posterior_appendix}
\end{equation}

The negative log-posterior, often referred to as the potential energy function $U(\phi) = -\log p(\phi | \mathcal{D}_{\text{batch}})$, can be written as:
\begin{equation}
    U(\phi) = \frac{B}{2\sigma^2} \mathcal{L}_{\text{MSE}}(\phi) + \frac{\nu}{2}\|\phi\|^2 + \text{const}.
    \label{eq:potential_energy}
\end{equation}

Equation \ref{eq:potential_energy} demonstrates that minimizing the MSE loss (with L2 regularization) is equivalent to finding the Maximum A Posteriori (MAP) estimate of $\phi$. However, a single MAP point estimate fails to capture epistemic uncertainty, the model's lack of knowledge in unexplored regions. To robustly estimate this uncertainty for safe decision-making, we must capture the diversity of plausible Q-functions by sampling from the full posterior distribution $p(\phi | \mathcal{D}_{\text{batch}})$.

Since exact sampling is analytically intractable for deep networks, we employ Stochastic Gradient Langevin Dynamics (SGLD). By injecting Gaussian noise into the gradient updates, SGLD approximates the Langevin diffusion process which has the posterior as its stationary distribution:
\begin{equation}
    \phi_{t+1} = \phi_t - \eta \nabla_\phi U(\phi_t) + \sqrt{2\eta T^{-1}} \, \xi_t, \quad \xi_t \sim \mathcal{N}(0, I).
\end{equation}

In SL-SAC, we use the adaptive variant (aSGLD) described in Eq.~\ref{eq:asgld_update}, which preconditions the gradient to handle the pathological curvature of deep networks while maintaining the Bayesian interpretation~\citep{dauphin2014identifying}. This allows each member of the ensemble to approximate a sample from the posterior $p(\phi | \mathcal{D})$, providing diverse Q-function estimates that reflect the underlying epistemic uncertainty, thereby yielding a more robust aggregated value estimate using Bayesian Model Averaging.

The standard derivation of adaptive SGLD (Eq.~\ref{eq:original_asgld}) requires scaling the noise term by the preconditioner matrix to strictly satisfy the detailed balance condition and sample from the exact posterior $p(\phi|\mathcal{D}_{\text{batch}})$. However, as empirically demonstrated in Appendix~\ref{subsec:noise_types_ablation}, adhering strictly to this Bayesian formulation results in suboptimal asymptotic returns and transient safety violations during the early phases of training. By utilizing adaptive drift with isotropic noise, SL-SAC achieves a superior balance between convergence speed and robust exploration. We therefore interpret the resulting stationary distribution as an approximation to the true posterior, prioritizing practical safety and reward maximization over exact MCMC sampling.

\section{Theoretical Proof}

\subsection{Contraction of the Distributional Bellman Operator}

Let $Z^\pi(s,a)$ be the random variable representing the return of policy $\pi$ starting at $(s,a)$. The distributional Bellman operator $\mathcal{T}^\pi$ is defined as:
\begin{equation}
    \mathcal{T}^\pi Z(s,a) \overset{D}{=} c(s,a) + \gamma Z(s', a'), \quad s' \sim P(\cdot|s,a), a' \sim \pi(\cdot|s').
\end{equation}
We measure the distance between distributions using the $p$-Wasserstein metric $W_p$.

\begin{lemma}[Contraction in Wasserstein Metric~\citep{bellemare2017distributional}]
\label{lemma:contraction}
For a fixed policy $\pi$, the distributional Bellman operator $\mathcal{T}^\pi$ is a $\gamma$-contraction in the maximal $p$-Wasserstein metric $\bar{d}_p$:
\begin{equation}
    \bar{d}_p(\mathcal{T}^\pi Z_1, \mathcal{T}^\pi Z_2) \leq \gamma \bar{d}_p(Z_1, Z_2),
\end{equation}
where $\bar{d}_p(Z_1, Z_2) = \sup_{s,a} W_p(Z_1(s,a), Z_2(s,a))$. Consequently, applying $\mathcal{T}^\pi$ iteratively converges to the unique fixed-point distribution $Z^\pi$.
\end{lemma}

\subsection{Sensitivity of CVaR to Distributional Error}

\begin{theorem}
    \label{thm:cvar_error_appendix}
    Let $Z^\pi(s,a)$ be the true cost return distribution and $Z_\psi(s,a)$ be the IQN approximation. If the IQN network satisfies
    \begin{equation}
    \mathbb{E}_{\tau \sim \text{Unif}[0,1]} \left[ \left| F^{-1}_{Z^\pi(s,a)}(\tau) - Z_\psi(s,a; \tau) \right|^2 \right] \leq \delta^2,
    \end{equation}
    then for CVaR at level $\epsilon$:
    \begin{equation}
    \left| \text{CVaR}_\epsilon(Z^\pi(s,a)) - \text{CVaR}_\epsilon(Z_\psi(s,a)) \right| \leq \frac{\delta}{\sqrt{\epsilon}}.
    \end{equation}
\end{theorem}

\begin{proof}
The CVaR at level $\epsilon$ is defined via the quantile function as:
\begin{equation}
\text{CVaR}_\epsilon(Z) = \frac{1}{\epsilon} \int_{1-\epsilon}^1 F^{-1}_Z(\tau) \, d\tau.
\end{equation}

For the IQN approximation, we denote $Z_\psi(s,a; \tau)$ as the network's estimate of the $\tau$-th quantile. We bound the absolute difference in CVaR estimates:
\begin{align*}
&\left| \text{CVaR}_\epsilon(Z^\pi(s,a)) - \text{CVaR}_\epsilon(Z_\psi(s,a)) \right| \\
&= \left| \frac{1}{\epsilon} \int_{1-\epsilon}^1 F^{-1}_{Z^\pi(s,a)}(\tau) \, d\tau - \frac{1}{\epsilon} \int_{1-\epsilon}^1 Z_\psi(s,a; \tau) \, d\tau \right| \\
&= \frac{1}{\epsilon} \left| \int_{1-\epsilon}^1 \left( F^{-1}_{Z^\pi(s,a)}(\tau) - Z_\psi(s,a; \tau) \right) d\tau \right| \\
&\leq \frac{1}{\epsilon} \int_{1-\epsilon}^1 \left| F^{-1}_{Z^\pi(s,a)}(\tau) - Z_\psi(s,a; \tau) \right| d\tau.
\end{align*}

Applying the Cauchy-Schwarz inequality:
\begin{align*}
&\leq \frac{1}{\epsilon} \sqrt{\int_{1-\epsilon}^1 1 \, d\tau} \cdot \sqrt{\int_{1-\epsilon}^1 \left| F^{-1}_{Z^\pi(s,a)}(\tau) - Z_\psi(s,a; \tau) \right|^2 d\tau} \\
&= \frac{1}{\epsilon} \sqrt{\epsilon} \cdot \sqrt{\int_{1-\epsilon}^1 \left| F^{-1}_{Z^\pi(s,a)}(\tau) - Z_\psi(s,a; \tau) \right|^2 d\tau} \\
&= \frac{1}{\sqrt{\epsilon}} \sqrt{\int_{1-\epsilon}^1 \left| F^{-1}_{Z^\pi(s,a)}(\tau) - Z_\psi(s,a; \tau) \right|^2 d\tau}.
\end{align*}

Since $[1-\epsilon, 1] \subseteq [0,1]$ and the integrand is non-negative:
\begin{equation}
\int_{1-\epsilon}^1 \left| F^{-1}_{Z^\pi(s,a)}(\tau) - Z_\psi(s,a; \tau) \right|^2 d\tau \leq \int_0^1 \left| F^{-1}_{Z^\pi(s,a)}(\tau) - Z_\psi(s,a; \tau) \right|^2 d\tau \leq \delta^2.
\end{equation}

Therefore:
\begin{equation}
\left| \text{CVaR}_\epsilon(Z^\pi(s,a)) - \text{CVaR}_\epsilon(Z_\psi(s,a)) \right| \leq \frac{\delta}{\sqrt{\epsilon}}.
\end{equation}
\end{proof}

\begin{figure}[t]
    \centering
    % Row 1, Column 1
    \begin{minipage}{0.24\textwidth}
        \centering
        \includegraphics[width=\linewidth]{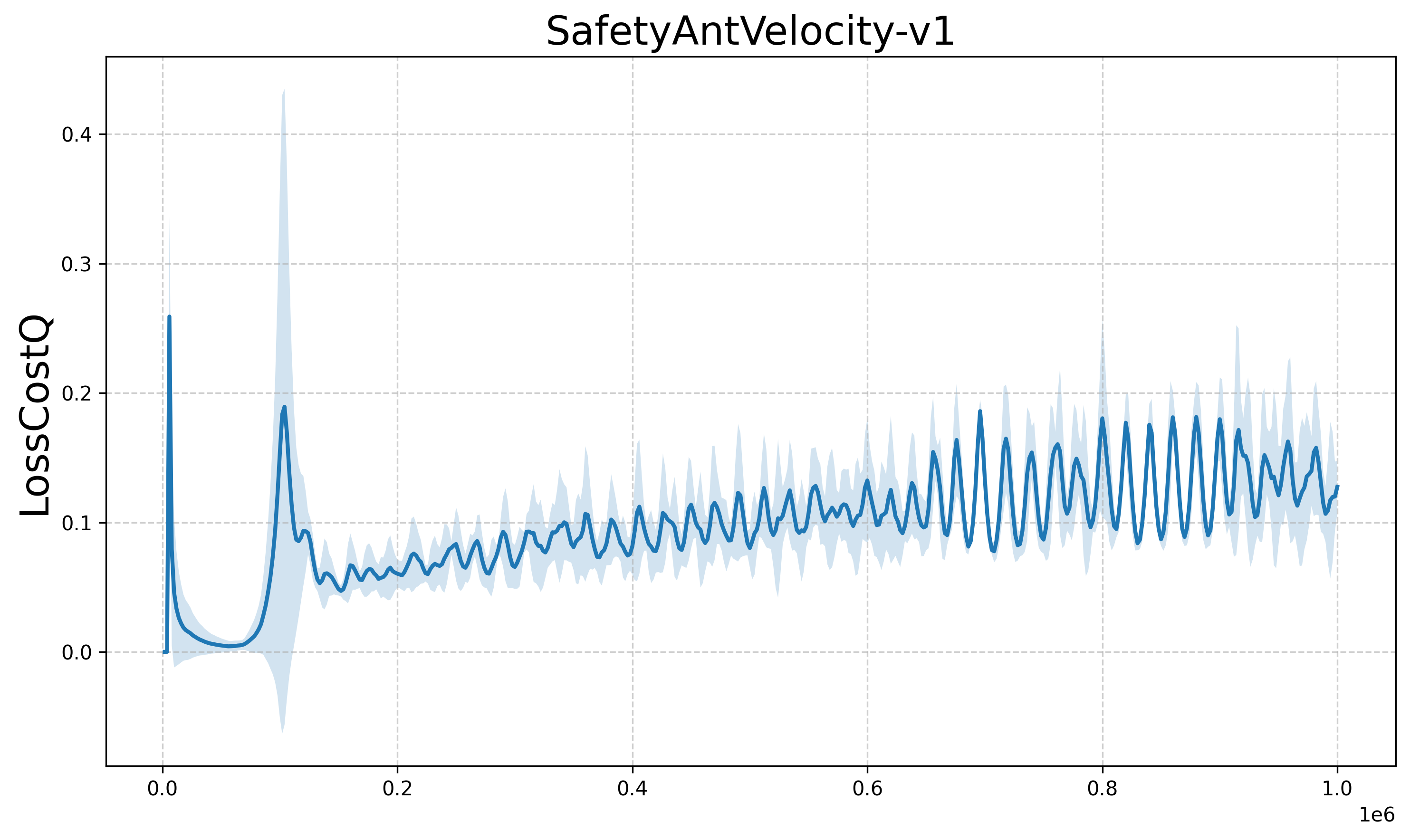}
    \end{minipage}
    \hfill
    % Row 1, Column 2
    \begin{minipage}{0.24\textwidth}
        \centering
        \includegraphics[width=\linewidth]{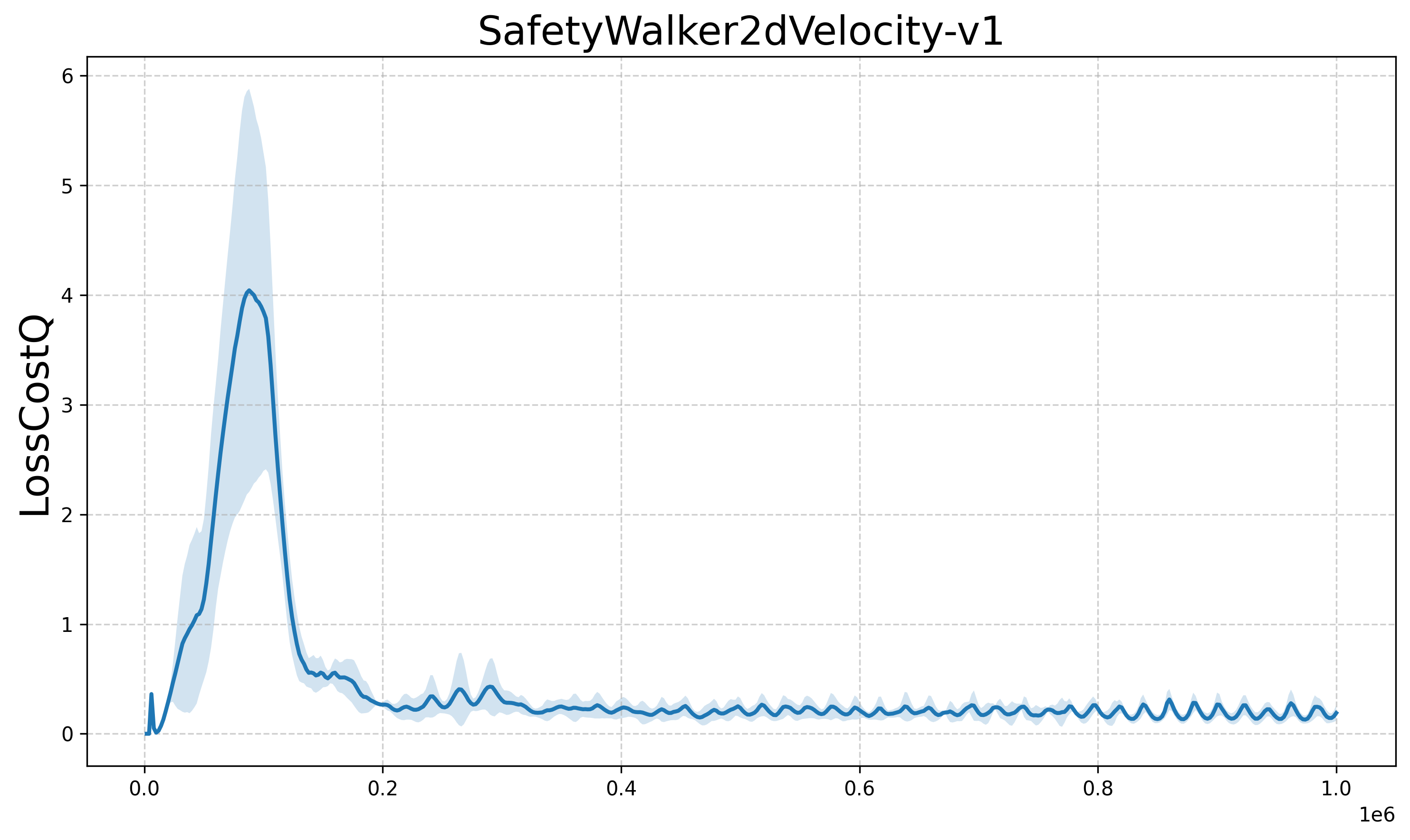}
    \end{minipage}
    \hfill
    % Row 1, Column 3
    \begin{minipage}{0.24\textwidth}
        \centering
        \includegraphics[width=\linewidth]{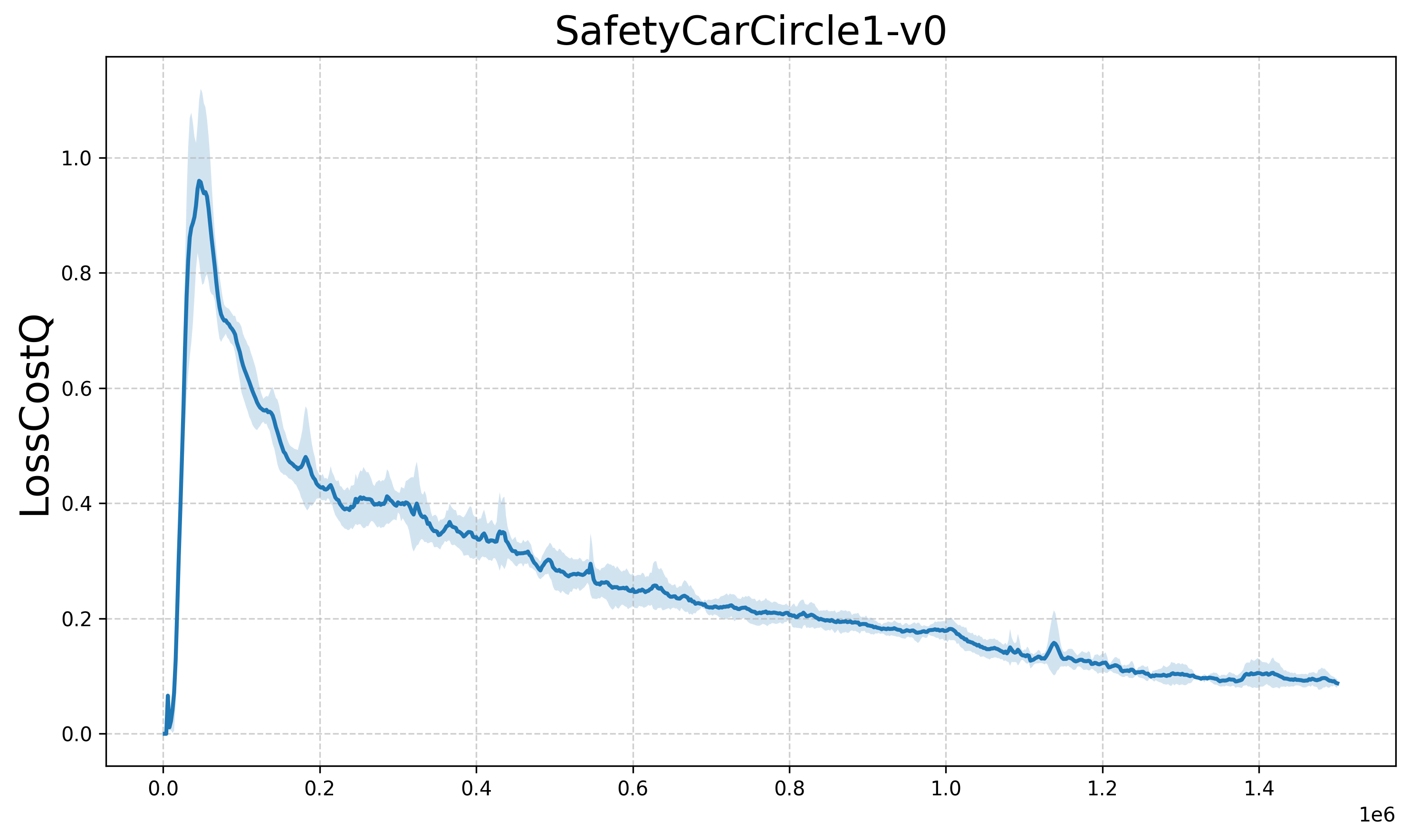}
    \end{minipage}
    \hfill
    % Row 1, Column 4
    \begin{minipage}{0.24\textwidth}
        \centering
        \includegraphics[width=\linewidth]{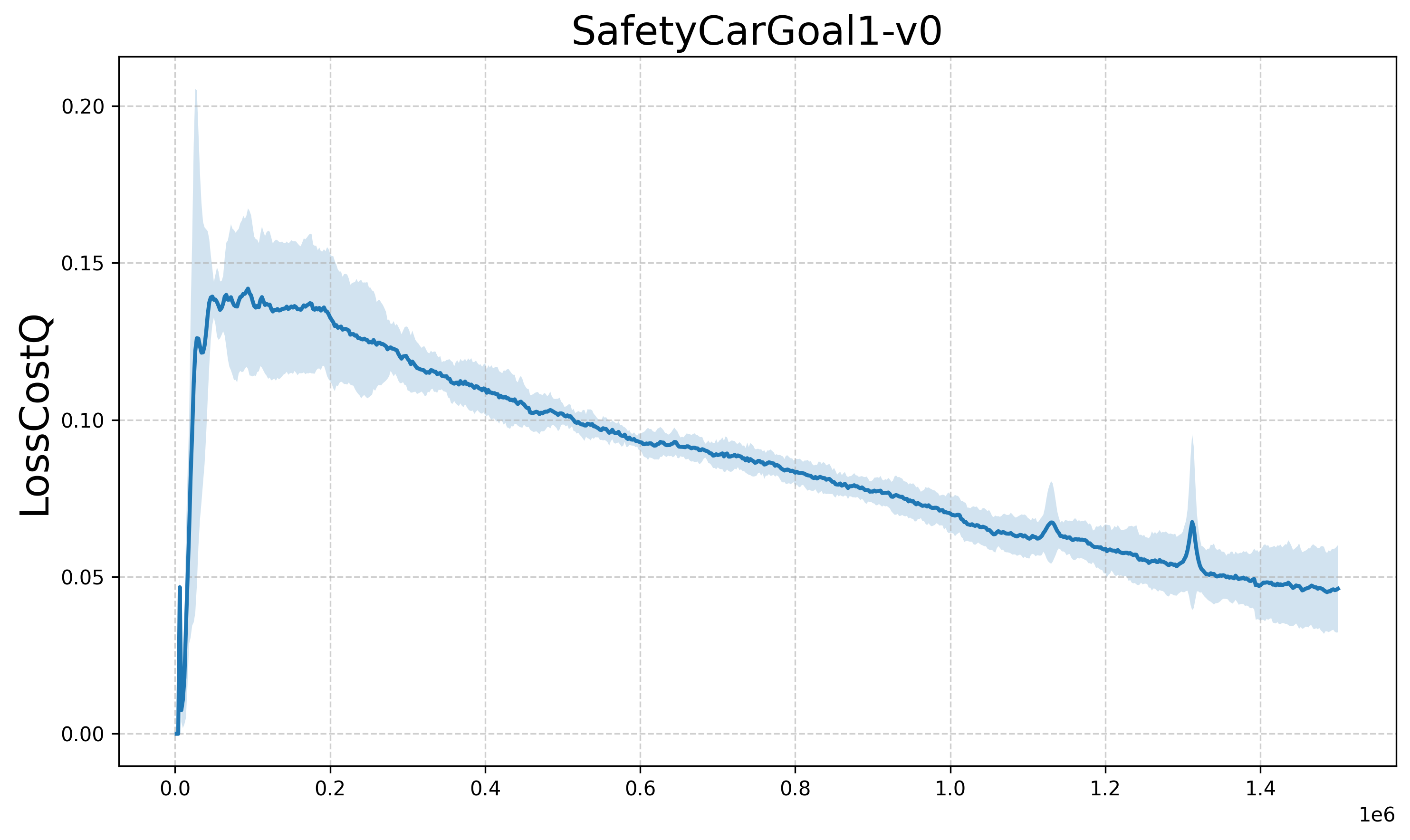}
    \end{minipage}
    
    \caption{\textbf{Convergence of Cost Critic Approximation.} Evolution of the Quantile Huber Loss (averaged over 5 seeds) during training.}
    \label{fig:quantile_loss}
\end{figure}

\begin{corollary}[Safety‑Margin Adjustment for CVaR Constraints]
\label{cor:safety_bound}
Assume the finite-sample error in Eq. \eqref{eq:cvar_estimate} is negligible (i.e., $N$ is large). If the learned policy $\pi$ satisfies $\widehat{\text{CVaR}}_\epsilon(Z_\psi) \leq \beta$, and the distributional critic achieves mean squared error $\delta^2 = \mathbb{E}_{\tau \sim \text{Unif}[0,1]} [ | F^{-1}_{Z^\pi}(\tau) - Z_\psi(\tau) |^2 ]$, then by Theorem \ref{thm:cvar_error}, the true CVaR satisfies:
\begin{equation}
    \text{CVaR}_\epsilon(Z^\pi) \leq \beta + \frac{\delta}{\sqrt{\epsilon}}.
\end{equation}
Thus, to guarantee $\text{CVaR}_\epsilon(Z^\pi) \leq \beta$, it suffices to set $\beta' = \beta - \frac{\delta}{\sqrt{\epsilon}}$ and enforce $\widehat{\text{CVaR}}_\epsilon(Z_\psi) \leq \beta'$, provided $\beta > \frac{\delta}{\sqrt{\epsilon}}$.
\end{corollary}

Corollary~\ref{cor:safety_bound} suggests using a tightened threshold $\beta' = \beta - \delta/\sqrt{\epsilon}$ to theoretically guarantee safety against approximation errors. However, in our experiments, we utilize the target threshold $\beta$ directly without the adjustment term. This design choice is justified by two factors: (1) The theoretical bound is a worst-case estimate, and calculating the exact upper bound $\delta$ requires access to the unknown ground-truth distribution $Z^\pi$; (2) Empirically, the approximation error $\delta$ remains sufficiently small throughout training. As illustrated in Figure~\ref{fig:quantile_loss}, the Quantile Huber Loss, which serves as a proxy for the approximation error, rapidly converges to a low magnitude across diverse environments. This indicates that the distributional critic $Z_\psi$ closely tracks $Z^\pi$, allowing the adaptive Lagrangian to maintain safety without requiring an artificially conservative threshold $\beta'$.

\begin{corollary}
    \label{cor:cvar_vs_expected_appendix}
    Consider a policy $\pi$ with non-negative true cost distribution $Z^\pi$ and IQN approximation $Z_\psi$ achieving mean squared quantile error $\delta^2$ as in Theorem~\ref{thm:cvar_error}. For a target cost limit $\beta$ and risk level $\epsilon \in (0,1)$:
    
    \begin{enumerate}
        \item \textbf{CVaR-based constraint:} If $\widehat{\mathrm{CVaR}}_\epsilon(Z_\psi) \leq \beta$, then by Corollary \ref{cor:safety_bound}:
        \begin{equation}
            \mathrm{CVaR}_\epsilon(Z^\pi) \leq \beta + \frac{\delta}{\sqrt{\epsilon}}.
            % \label{eq:cvar_constraint_bound}
        \end{equation}
        
        \item \textbf{Expected cost constraint:} If $\mathbb{E}[Z_\psi] \leq \beta$, then:
        \begin{equation}
            \mathrm{CVaR}_\epsilon(Z^\pi) \leq \frac{\beta + \delta}{\epsilon}.
            % \label{eq:expected_constraint_bound}
        \end{equation}
    \end{enumerate}
    
    The CVaR-based constraint yields a strictly tighter bound on true tail risk whenever
    \begin{equation}
        \beta + \frac{\delta}{\sqrt{\epsilon}} < \frac{\beta + \delta}{\epsilon},
        \label{eq:tighter_condition}
    \end{equation}
    which holds for all $\beta > 0$, $\delta > 0$, and $\epsilon \in (0,1)$.
\end{corollary}

\begin{proof}
\textbf{Part 1:} Follows directly from Theorem~\ref{thm:cvar_error} and Corollary~\ref{cor:safety_bound}.

\textbf{Part 2:} The proof proceeds in three steps.

\textit{Step 1: MSE bounds the error in expectation.} 
By the Cauchy–Schwarz inequality,
\begin{align*}
    \left|\mathbb{E}[Z^\pi] - \mathbb{E}[Z_\psi]\right| 
    &= \left| \int_0^1 \bigl( F^{-1}_{Z^\pi}(\tau) - Z_\psi(\tau) \bigr) \, d\tau \right| \\
    &\leq \int_0^1 \bigl| F^{-1}_{Z^\pi}(\tau) - Z_\psi(\tau) \bigr| \, d\tau \\
    &\leq \sqrt{\int_0^1 1 \, d\tau} \; \sqrt{\int_0^1 \bigl| F^{-1}_{Z^\pi}(\tau) - Z_\psi(\tau) \bigr|^2 \, d\tau}
    = \delta .
\end{align*}

\textit{Step 2: Worst‑case CVaR for a given expectation.}
For any non‑negative random variable $Z$ with $\mathbb{E}[Z] = \mu$, the maximal possible $\mathrm{CVaR}_\epsilon(Z)$ is attained by the two‑point distribution that places mass $(1-\epsilon)$ at $0$ and mass $\epsilon$ at $\mu/\epsilon$. Consequently,
\begin{equation}
    \mathrm{CVaR}_\epsilon(Z) \leq \frac{\mathbb{E}[Z]}{\epsilon}.
    \label{eq:worst_case_cvar}
\end{equation}

\textit{Step 3: Combining the bounds.}
If $\mathbb{E}[Z_\psi] \leq \beta$, Step 1 implies $\mathbb{E}[Z^\pi] \leq \mathbb{E}[Z_\psi] + \delta \leq \beta + \delta$. 
Applying the worst‑case bound \eqref{eq:worst_case_cvar} to $Z^\pi$ gives
\begin{equation*}
    \mathrm{CVaR}_\epsilon(Z^\pi) \leq \frac{\mathbb{E}[Z^\pi]}{\epsilon} \leq \frac{\beta + \delta}{\epsilon}.
\end{equation*}

To show the CVaR-based bound is strictly tighter, we verify:
\begin{align*}
    \frac{\beta + \delta}{\epsilon} - \left(\beta + \frac{\delta}{\sqrt{\epsilon}}\right) &= \frac{\beta + \delta}{\epsilon} - \beta - \frac{\delta}{\sqrt{\epsilon}} \\
    &= \beta\left(\frac{1}{\epsilon} - 1\right) + \delta\left(\frac{1}{\epsilon} - \frac{1}{\sqrt{\epsilon}}\right) \\
    &= \beta\left(\frac{1-\epsilon}{\epsilon}\right) + \delta\left(\frac{1 - \sqrt{\epsilon}}{\epsilon}\right).
\end{align*}
Since $\beta > 0$, $\delta > 0$, and $\epsilon \in (0,1)$, both terms are strictly positive, establishing strict inequality.

\end{proof}

\begin{theorem}[CVaR Provides Stronger Constraint Violation Signals]
\label{thm:cvar_signal_strength}
Let $\{C_i\}_{i=1}^N$ be a sliding window of $N$ recent episodic costs. Let $\hat{\beta}_\mathrm{CVaR} = \mathrm{CVaR}_\epsilon(\{C_i\})$ and $\hat{\beta}_\mathrm{EXP} = \mathbb{E}[\{C_i\}]$ be empirical estimates of CVaR and expected cost, respectively. For any safety threshold $\beta > 0$ and risk level $\epsilon \in (0,1)$, the constraint violation signals satisfy:

\textbf{Signal strength:} The CVaR-based violation signal is always algebraically greater than or equal to the expected-cost signal:
\begin{equation}
    \hat{\beta}_\mathrm{CVaR} - \beta \geq \hat{\beta}_\mathrm{EXP} - \beta,
    \label{eq:signal_strength}
\end{equation}
with strict inequality whenever the costs in the window are not all identical. Specifically, the relationship is given by the identity:
\begin{equation}
    (\hat{\beta}_\mathrm{CVaR} - \beta) - (\hat{\beta}_\mathrm{EXP} - \beta) = \left(1 - \frac{k}{N}\right) (\hat{\beta}_\mathrm{CVaR} - \hat{\beta}_\mathrm{body}) \geq 0,
    \label{eq:signal_decomposition}
\end{equation}
where $k = \lceil \epsilon N \rceil$ is the number of tail samples and $\hat{\beta}_\mathrm{body}$ is the mean of the lower $(N-k)$ ordered costs.

Consequently, in the Lagrange multiplier update $\lambda_{t+1} = \max(0, \lambda_t + \eta(\hat{\beta} - \beta))$, the CVaR-based update always yields a larger positive adjustment (or a smaller negative decay) to $\lambda$ compared to the expected cost. When costs exceed the threshold ($\hat{\beta} > \beta$), CVaR provides a stronger penalty increase; when below threshold ($\hat{\beta} < \beta$), CVaR results in a slower safety relaxation.
\end{theorem}

\begin{proof}
Let $C_{(1)} \leq C_{(2)} \leq \cdots \leq C_{(N)}$ be the order statistics of the costs. Let $k = \lceil \epsilon N \rceil$ denote the number of samples in the risk tail. The empirical CVaR and body mean are defined as:
\[
\hat{\beta}_\mathrm{CVaR} = \frac{1}{k} \sum_{i=N-k+1}^{N} C_{(i)},
\quad
\hat{\beta}_\mathrm{body} = \frac{1}{N-k} \sum_{i=1}^{N-k} C_{(i)}.
\]
The expected cost $\hat{\beta}_\mathrm{EXP}$ is the arithmetic mean of all samples, which can be decomposed into a weighted sum of the tail and body means:
\[
\hat{\beta}_\mathrm{EXP} = \frac{1}{N}\sum_{i=1}^{N} C_{(i)} 
= \frac{1}{N} \left( \sum_{i=1}^{N-k} C_{(i)} + \sum_{i=N-k+1}^{N} C_{(i)} \right)
= \frac{N-k}{N}\hat{\beta}_\mathrm{body} + \frac{k}{N}\hat{\beta}_\mathrm{CVaR}.
\]
We analyze the difference between the violation signals:
\begin{align*}
(\hat{\beta}_\mathrm{CVaR} - \beta) - (\hat{\beta}_\mathrm{EXP} - \beta) &= \hat{\beta}_\mathrm{CVaR} - \hat{\beta}_\mathrm{EXP} \\
&= \hat{\beta}_\mathrm{CVaR} - \left[ \left(1 - \frac{k}{N}\right)\hat{\beta}_\mathrm{body} + \frac{k}{N}\hat{\beta}_\mathrm{CVaR} \right] \\
&= \left(1 - \frac{k}{N}\right)\hat{\beta}_\mathrm{CVaR} - \left(1 - \frac{k}{N}\right)\hat{\beta}_\mathrm{body} \\
&= \left(1 - \frac{k}{N}\right) \left( \hat{\beta}_\mathrm{CVaR} - \hat{\beta}_\mathrm{body} \right).
\end{align*}
Since $0 < \epsilon < 1$, we have $N > k$, so the term $(1 - k/N)$ is strictly positive.
Furthermore, because the costs are ordered $C_{(1)} \leq \dots \leq C_{(N)}$, the average of the top $k$ values ($\hat{\beta}_\mathrm{CVaR}$) must be greater than or equal to the average of the bottom $N-k$ values ($\hat{\beta}_\mathrm{body}$), i.e., $\hat{\beta}_\mathrm{CVaR} \geq \hat{\beta}_\mathrm{body}$. Equality holds if and only if all costs are identical ($C_{(1)} = C_{(N)}$).

Therefore, the difference is non-negative:
\[
\hat{\beta}_\mathrm{CVaR} - \beta \geq \hat{\beta}_\mathrm{EXP} - \beta.
\]
\end{proof}

\begin{theorem}
    \label{app:cvar_gpd_violation_appendix}
    Let $Z^\pi(s,a)$ be the random variable representing the true cumulative discounted cost return of policy $\pi$ starting at state $s$ and taking action $a$. Fix a confidence level $\epsilon \in (0,1)$ and a safety threshold $\beta \in \mathbb{R}$.
    Let $\mathrm{VaR}_\epsilon(Z^\pi) = \inf \{ z \in \mathbb{R} : \Pr(Z^\pi \le z) \ge \epsilon \}$ denote the Value-at-Risk.
    
    Assume that the tail of the true cost distribution follows a Generalized Pareto Distribution (GPD) with shape parameter $\nu$ and scale parameter $\sigma$. Specifically:
    \begin{enumerate}
        \item Let $u := \mathrm{VaR}_\epsilon(Z^\pi)$. The exceedance follows a GPD: $Z^\pi(s,a) - u \mid Z^\pi(s,a) > u \sim \mathrm{GPD}(\nu, \sigma)$.
        \item The shape parameter satisfies $0 \le \nu < 1$ (ensuring a finite mean).
        \item The true risk satisfies the safety constraint: $\mathrm{CVaR}_\epsilon(Z^\pi(s,a)) \le \beta$.
    \end{enumerate}
    Then, the probability of the true cost violating the safety threshold $\beta$ is bounded by:
    \[
    \Pr(Z^\pi(s,a) > \beta) \le (1-\epsilon) \, \gamma(\nu),
    \]
    where the scaling factor $\gamma(\nu)$ depends on the tail heaviness:
    \[
    \gamma(\nu) = 
    \begin{cases}
    e^{-1}, & \text{if } \nu = 0 \text{ (Exponential tail)}, \\[4pt]
    (1-\nu)^{1/\nu}, & \text{if } 0 < \nu < 1 \text{ (Heavy tail)}.
    \end{cases}
    \]
\end{theorem}

\begin{proof}
Let $Z^\pi$ denote $Z^\pi(s,a)$ for brevity. Let $u := \mathrm{VaR}_\epsilon(Z^\pi)$ such that $\Pr(Z^\pi > u) = 1-\epsilon$.
Define the exceedance random variable $Y := Z^\pi - u \mid Z^\pi > u$, which follows $\mathrm{GPD}(\nu, \sigma)$. We analyze the bound based on the shape parameter $\nu$.

\textbf{Case 1: Light Tail ($\nu = 0$).} 
The GPD reduces to an Exponential distribution with mean $\sigma$.
From the definition of CVaR, we have $\mathrm{CVaR}_\epsilon(Z^\pi) = u + \mathbb{E}[Y] = u + \sigma$.
Imposing the safety constraint $\mathrm{CVaR}_\epsilon(Z^\pi) \le \beta$ implies $u + \sigma \le \beta$, or equivalently $\beta - u \ge \sigma$.
The survival function is $\Pr(Y > y) = \exp(-y/\sigma)$ for $y \ge 0$. Thus:
\[
\Pr(Z^\pi > \beta) = (1-\epsilon) \Pr(Y > \beta - u) = (1-\epsilon) \exp\left(-\frac{\beta - u}{\sigma}\right).
\]
Since $\beta - u \ge \sigma$, the term inside the exponential is $\le -1$. Therefore:
\[
\Pr(Z^\pi > \beta) \le (1-\epsilon) e^{-1}.
\]

\textbf{Case 2: Heavy Tail ($0 < \nu < 1$).} 
The mean of the GPD is $\mathbb{E}[Y] = \sigma/(1-\nu)$.
The constraint $\mathrm{CVaR}_\epsilon(Z^\pi) \le \beta$ implies:
\[
u + \frac{\sigma}{1-\nu} \le \beta \implies \beta - u \ge \frac{\sigma}{1-\nu}.
\]
The survival function is $\Pr(Y > y) = \left(1 + \frac{\nu y}{\sigma}\right)^{-1/\nu}$ for $y \ge 0$.
Substituting the minimal gap $y = \frac{\sigma}{1-\nu}$:
\[
\Pr(Z^\pi > \beta) = (1-\epsilon) \left(1 + \frac{\nu(\beta - u)}{\sigma}\right)^{-1/\nu}
          \le (1-\epsilon) \left(1 + \frac{\nu}{\sigma} \cdot \frac{\sigma}{1-\nu}\right)^{-1/\nu}.
\]
Simplifying the bracketed term to $1/(1-\nu)$, we obtain:
\[
\Pr(Z^\pi > \beta) \le (1-\epsilon) \left( \frac{1}{1-\nu} \right)^{-1/\nu} = (1-\epsilon)(1-\nu)^{1/\nu}.
\]
\end{proof}

This theorem complements our earlier results on approximation error.  In Theorem~\ref{thm:cvar_error} and Corollary~\ref{cor:safety_bound}, we established that by minimizing the quantile regression error $\delta$, SL-SAC ensures that the approximated CVaR constraints effectively bound the true CVaR (i.e., $\mathrm{CVaR}_\epsilon(Z^\pi) \lesssim \beta$). Theorem~\ref{thm:cvar_gpd_violation} extends this guarantee to the physical world: provided that the true cost distribution exhibits standard heavy-tailed behavior (approximated by GPD) and the CVaR constraint is satisfied, the probability of constraint violation ($\Pr(Z^\pi > \beta)$) is strictly bounded.

\begin{figure*}[t]
    \centering
    % --- Row 2 ---
    \begin{subfigure}[b]{0.48\textwidth}
        \centering
        \includegraphics[width=\linewidth]{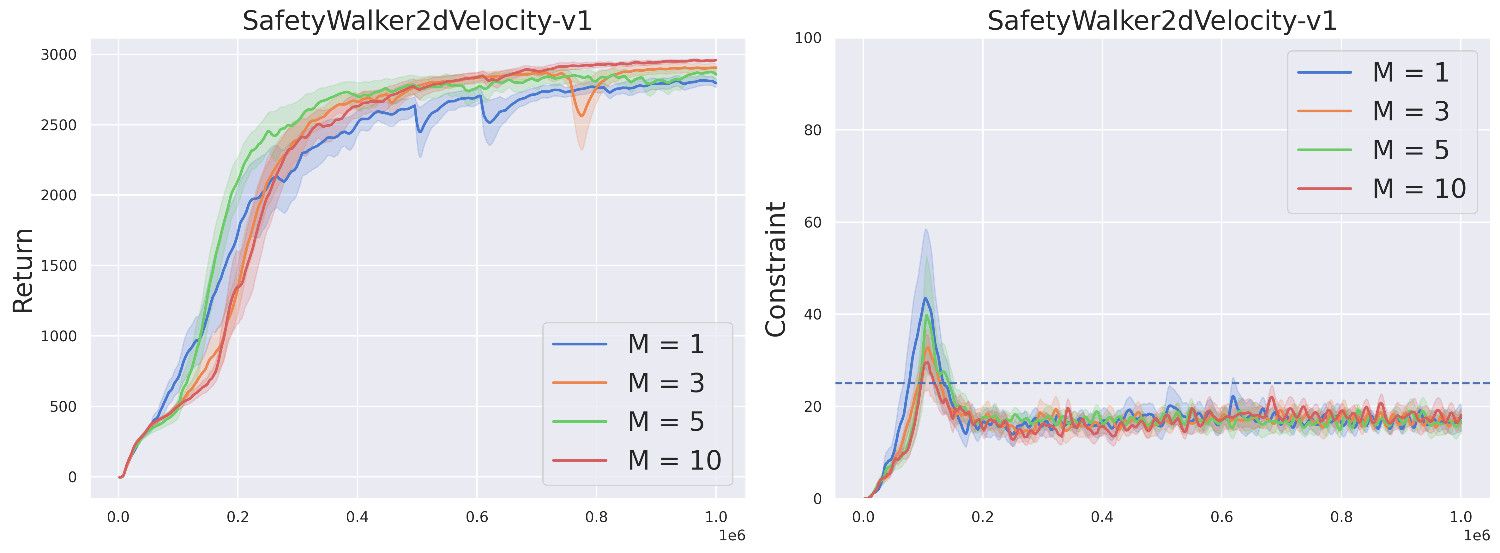}
        \caption{SafetyWalker2dVelocity-v1}
        \label{fig:walker_ensemble}
    \end{subfigure}
    \hfill
    \begin{subfigure}[b]{0.48\textwidth}
        \centering
        \includegraphics[width=\linewidth]{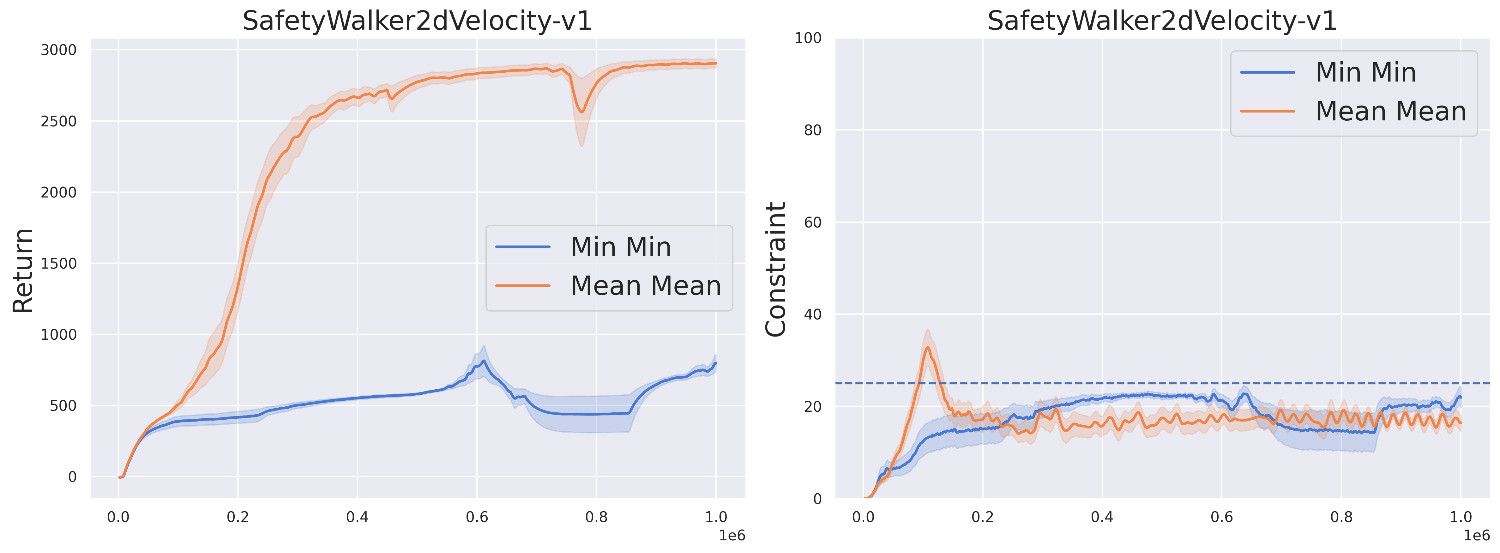}
        \caption{SafetyWalker2dVelocity-v1}
        \label{fig:walker_aggregation}
    \end{subfigure}
    
    \caption{\textbf{Ensemble size and aggregation strategy ablations.} (Left) Performance for $M \in \{1,3,5,10\}$ (ensemble size). (Right) Comparison of Min-Min vs. Mean-Mean aggregation.}
    \label{fig:ensemble_ablation}
\end{figure*}

\section{Ablation Studies}
\label{app:ablation_studies}

\subsection{Effect of Ensemble Size}
\label{subsec:ensemble_ablation}
We analyze the sensitivity of SL-SAC to the ensemble size $M$, evaluating configurations with $M \in \{1, 3, 5, 10\}$ on SafetyWalker2dVelocity-v1. As shown in Figure~\ref{fig:walker_ensemble}, increasing $M$ from 1 to 3 yields clear benefits. Beyond $M=3$, the curves for $M=5$ and $M=10$ exhibit very similar asymptotic performance and variance, indicating diminishing returns from larger ensembles while incurring higher computational cost. Consequently, we adopt $M=3$ as a practical trade-off between performance and efficiency in all main experiments.

\subsection{Ensemble Aggregation Strategy}
We examine two strategies for aggregating the ensemble of reward critics in both target construction and policy optimization. Recall that SL-SAC uses $M$ pairs of twin critics, where each pair $(Q^{(2m-1)}_\phi, Q^{(2m)}_\phi)$ computes a clipped Q-value via $\min(Q^{(2m-1)}_\phi, Q^{(2m)}_\phi)$. The aggregation strategy determines how these $M$ clipped values are combined:

\textbf{Mean-Mean:} The target for reward critic training uses the mean across pairs,
\begin{equation}
    y^{\text{reward}} = r + \gamma \mathbb{E}_{a' \sim \pi_{\theta'}(\cdot|s')} \left[\frac{1}{M} \sum_{m=1}^M \min(Q^{(2m-1)}_{\phi'}(s', a'), Q^{(2m)}_{\phi'}(s', a')) - \alpha \log \pi_{\theta'}(a'|s')\right],
\end{equation}
and the policy objective similarly maximizes the mean ensemble estimate $\bar{Q}_\phi(s,a) = \frac{1}{M} \sum_{m=1}^M \min(Q^{(2m-1)}_\phi, Q^{(2m)}_\phi)$.

\textbf{Min-Min:} Both the target and policy objective instead use the minimum across all pairs,
\begin{equation}
    y^{\text{reward}} = r + \gamma \left[\min_{m \in \{1,\ldots,M\}} \min(Q^{(2m-1)}_{\phi'}(s', a'), Q^{(2m)}_{\phi'}(s', a')) - \alpha \log \pi_{\theta'}(a'|s')\right].
\end{equation}

As shown in Figure~\ref{fig:walker_aggregation}, the Mean-Mean strategy outperforms Min-Min. In SafetyWalker2dVelocity-v1 Min-Min fails to learn an effective policy, remaining near the initial baseline performance throughout training. This suggests that taking the minimum across the ensemble introduces excessive pessimism that hinders exploration and policy improvement.

\begin{figure*}[t] % The * spans both columns. [t] tries to place it at top of page.
    \centering
    
    % --- Row 1 ---
    \begin{subfigure}[b]{0.48\textwidth}
        \centering
        % Replace with your first environment image
        \includegraphics[width=\linewidth]{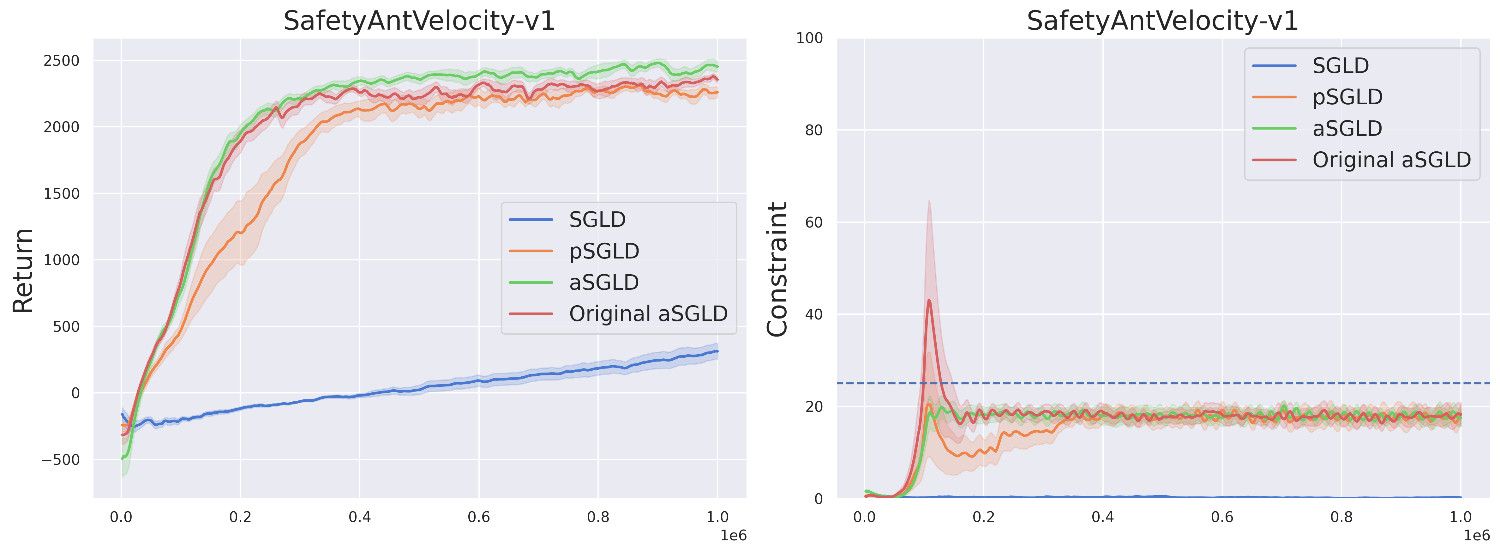} 
        \label{fig:ant_optimizers_ablation}
    \end{subfigure}
    \hfill % Adds flexible space between the two images
    \begin{subfigure}[b]{0.48\textwidth}
        \centering
        % Replace with your second environment image
        \includegraphics[width=\linewidth]{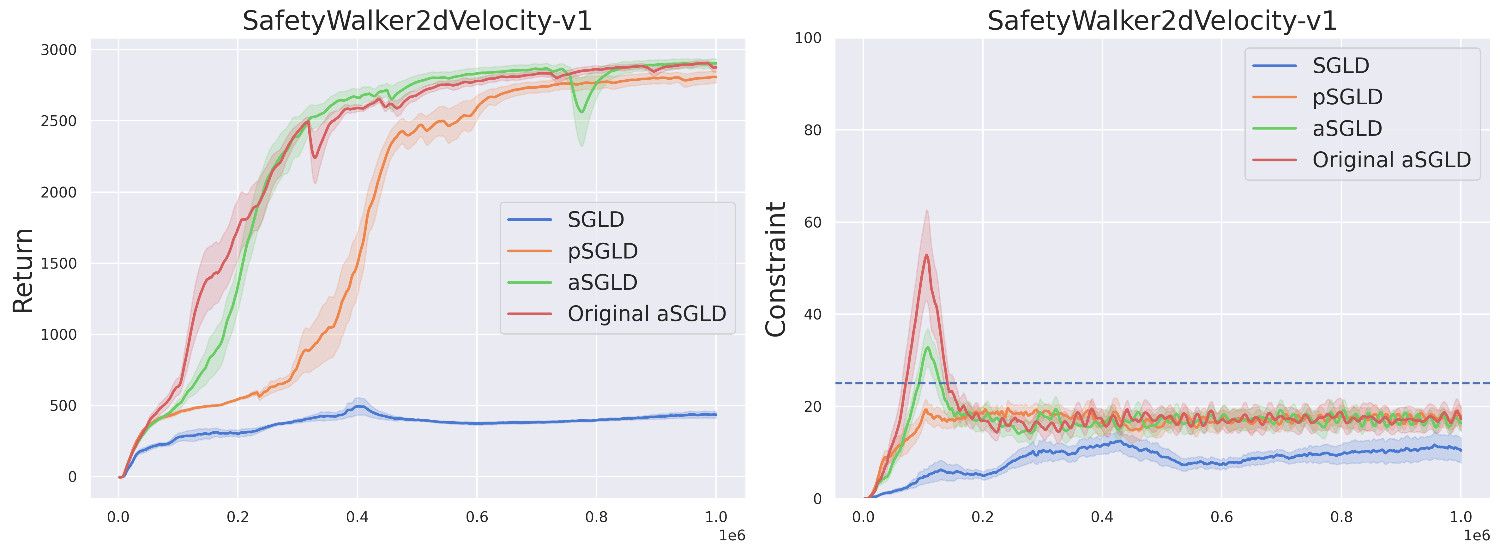}
        \label{fig:walker_optimizers_ablation}
    \end{subfigure}
    
    \caption{\textbf{Ablation on SGLD variants for reward critic optimization.} Comparison of vanilla SGLD, pSGLD (RMSProp preconditioner), aSGLD with isotropic noise, and fully preconditioned aSGLD.}
    \label{fig:ablation_optimizer}
\end{figure*}

\subsection{Noise Preconditioning in aSGLD}
\label{subsec:noise_types_ablation}
SL-SAC optimizes the ensemble of reward critics using an adaptive SGLD variant. A critical design choice is whether the injected Gaussian perturbation should be scaled by the preconditioner (anisotropic noise) or remain isotropic. To isolate the impact of this choice, we compare the aSGLD update rule used in SL-SAC with three standard baselines while keeping the actor and cost critic updates fixed.

\paragraph{Optimizer Variants.}
We analyze the update rules in the general form:
\begin{equation}
    \phi_{k+1} \leftarrow \phi_k - \eta \cdot \Delta(\phi_k) + \sqrt{2\eta T^{-1}} \cdot \Sigma(\phi_k) \, \xi_k, \quad \xi_k \sim \mathcal{N}(0, I_d),
    \label{eq:general_sgld}
\end{equation}
where $\Delta(\phi_k)$ represents the drift term and $\Sigma(\phi_k)$ determines the noise structure. We define $g_k = \nabla_\phi \mathcal{L}^{\text{reward}}(\phi_k)$ as the gradient, and $v_k$ as the exponential moving average of squared gradients. The preconditioner is denoted as $\zeta_k = \text{diag}(\sqrt{v_k + \varepsilon I})$. The variants are defined as follows:

\begin{description}
    \item[Vanilla SGLD:] Uses raw gradients and isotropic noise.
    \begin{equation}
        \Delta(\phi_k) = g_k, \quad \Sigma(\phi_k) = I_d.
    \end{equation}

    \item[pSGLD (RMSProp):] Uses a preconditioned drift but maintains isotropic noise.
    \begin{equation}
        \Delta(\phi_k) = \zeta_k^{-1} g_k, \quad \Sigma(\phi_k) = I_d.
    \end{equation}

    \item[Fully Preconditioned aSGLD:] Applies the preconditioner to both drift and noise, approximating sampling from a warped geometry.
    \begin{equation}
        \label{eq:original_asgld}
        \Delta(\phi_k) = \zeta_k^{-1} m_k, \quad \Sigma(\phi_k) = \zeta_k^{-1/2}.
    \end{equation}
    Note that $m_k$ is the Adam-style momentum of the gradient.

    \item[SL-SAC aSGLD:] Utilizes an adaptive preconditioner in the drift for fast convergence, but retains isotropic noise to ensure broad exploration of the posterior.
    \begin{equation}
        \Delta(\phi_k) = g_k + a(m_k \oslash \zeta_k), \quad \Sigma(\phi_k) = I_d.
    \end{equation}
\end{description}

We evaluate these variants on the SafetyAntVelocity-v1 and SafetyWalkerVelocity-v1 task using identical hyperparameters (Figure~\ref{fig:ablation_optimizer}). The results demonstrate that the SL-SAC aSGLD variant achieves the most stable learning dynamics. While preconditioning the drift (as in pSGLD and SL-SAC) is essential for avoiding the poor convergence seen in Vanilla SGLD, preconditioning the noise (Fully Preconditioned aSGLD) leads to transient cost spikes and instability. 

\begin{figure}[t]
    \centering
    % First Plot (Ant)
    \begin{minipage}{0.48\textwidth}
        \centering
        \includegraphics[width=\linewidth]{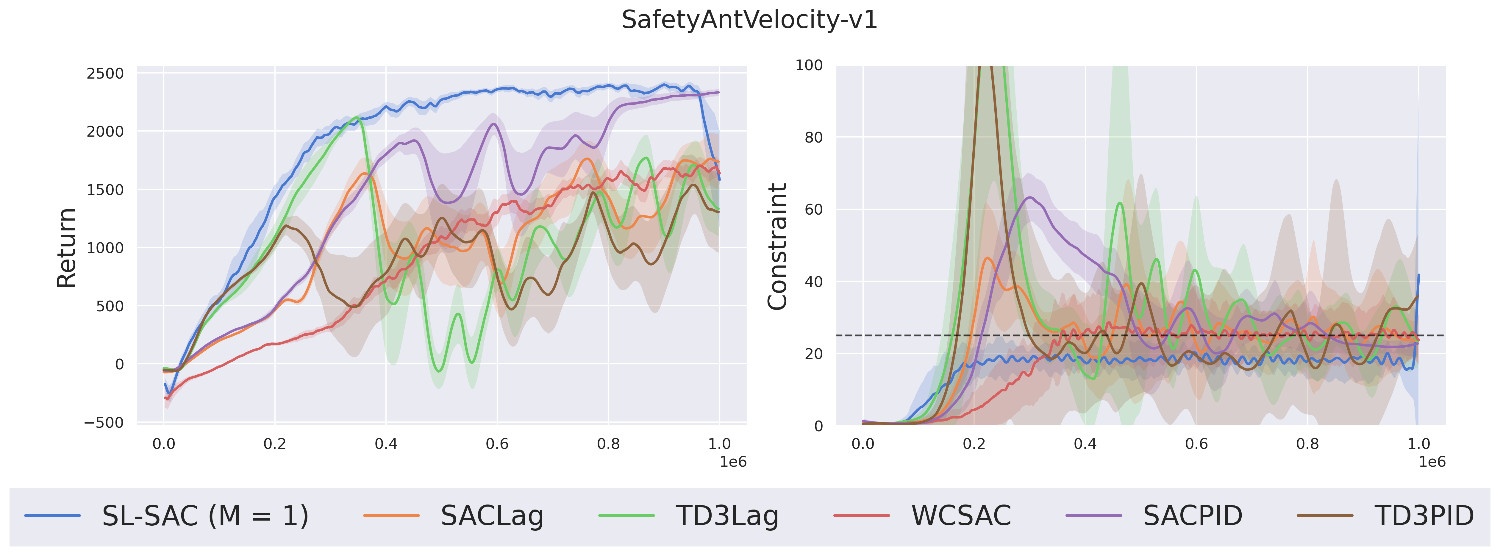} % Replace with actual filename
    \end{minipage}
    \hfill
    % Second Plot (Walker)
    \begin{minipage}{0.48\textwidth}
        \centering
        \includegraphics[width=\linewidth]{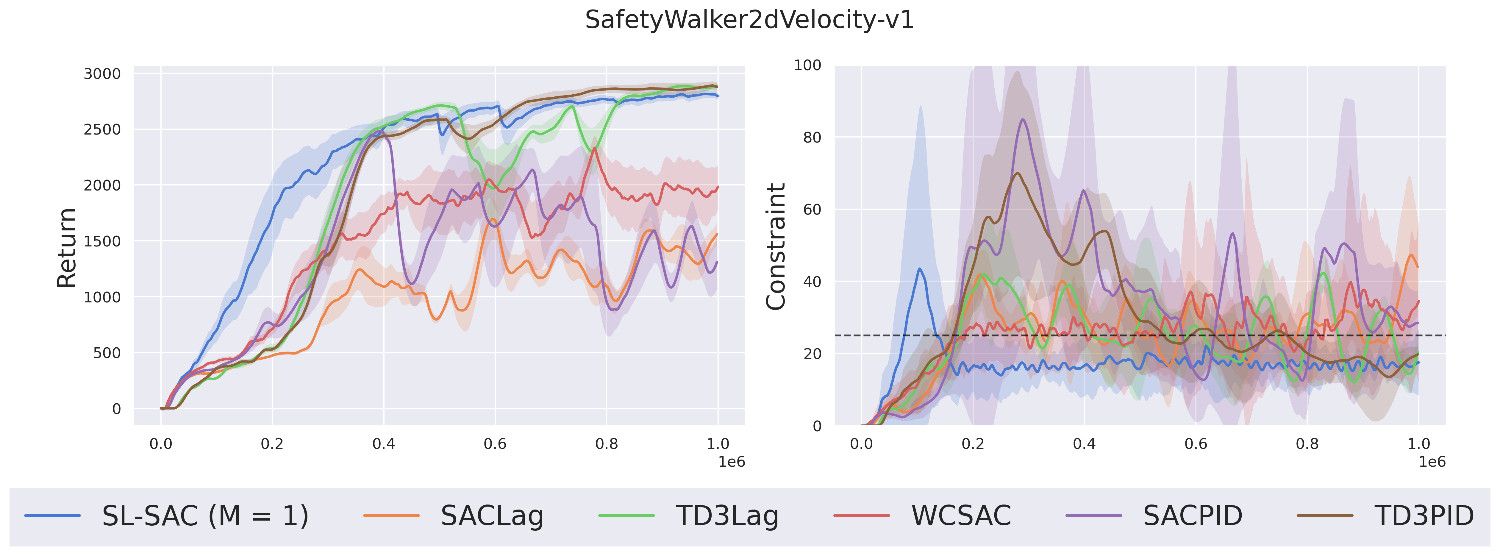} % Replace with actual filename
    \end{minipage}
    
    \caption{\textbf{Performance with $M=1$.} We compare SL-SAC restricted to a single twin-critic pair against standard baselines (which also use $M=1$).}
    \label{fig:m1_comparison}
\end{figure}

\subsection{Controlling for Model Capacity ($M=1$)}
\label{subsec:m1_capacity_ablation}

While we established the benefits of increasing the ensemble size in Appendix~\ref{subsec:ensemble_ablation}, it is crucial to verify that the performance gains of SL-SAC are not merely artifacts of increased model capacity. To isolate the algorithmic contributions, we evaluate SL-SAC using a single twin-critic pair ($M=1$), matching the exact parameter count of the standard baselines. As illustrated in Figure~\ref{fig:m1_comparison}, even with restricted capacity, SL-SAC ($M=1$)  outperforms in terms of asymptotic return while maintaining tighter constraint satisfaction. This confirms that the improvements are driven by the robust aSGLD optimization and distributional risk control mechanisms, rather than simply scaling the critic architecture.

While we analyzed the benefits of increasing the ensemble size in Appendix~\ref{subsec:ensemble_ablation}, it is important to verify that the performance of SL-SAC is not driven solely by increased model capacity. To isolate the algorithmic contributions, we evaluate a variant of SL-SAC using a single twin-critic pair ($M=1$), matching the parameter count of standard baselines. As illustrated in Figure~\ref{fig:m1_comparison}, even with restricted capacity, SL-SAC ($M=1$) remains highly effective, achieving asymptotic returns comparable to or exceeding baselines while upholding safety constraints. These results indicate that our method's efficacy stems primarily from the aSGLD optimization and distributional risk control mechanisms.

\begin{figure}[t]
    \centering
    \begin{minipage}{0.325\textwidth}
        \includegraphics[width=\linewidth]{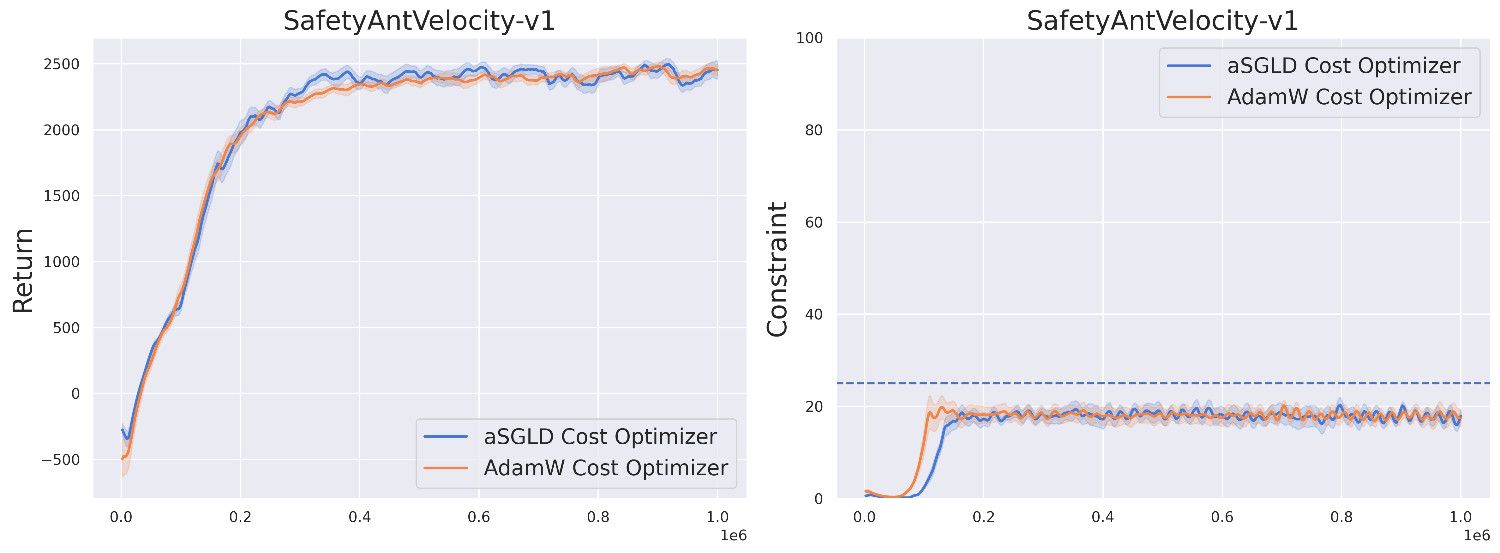}
    \end{minipage}
    \hfill
    \begin{minipage}{0.325\textwidth}
        \includegraphics[width=\linewidth]{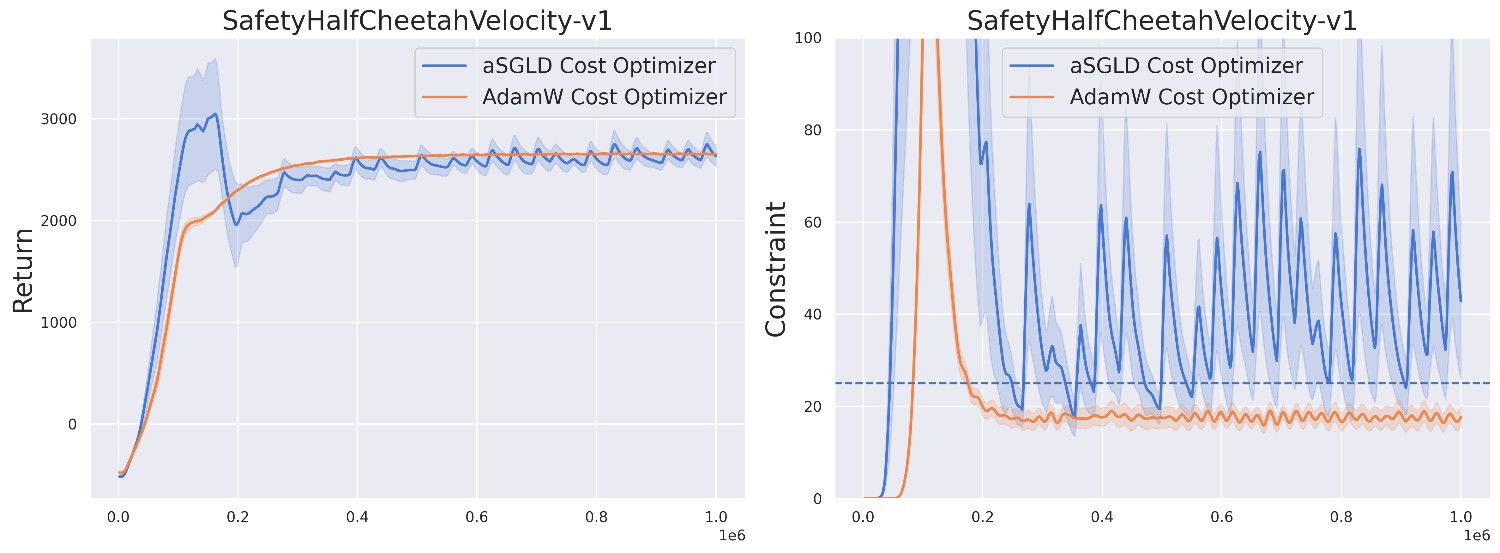}
    \end{minipage}
    \hfill
    \begin{minipage}{0.325\textwidth}
        \includegraphics[width=\linewidth]{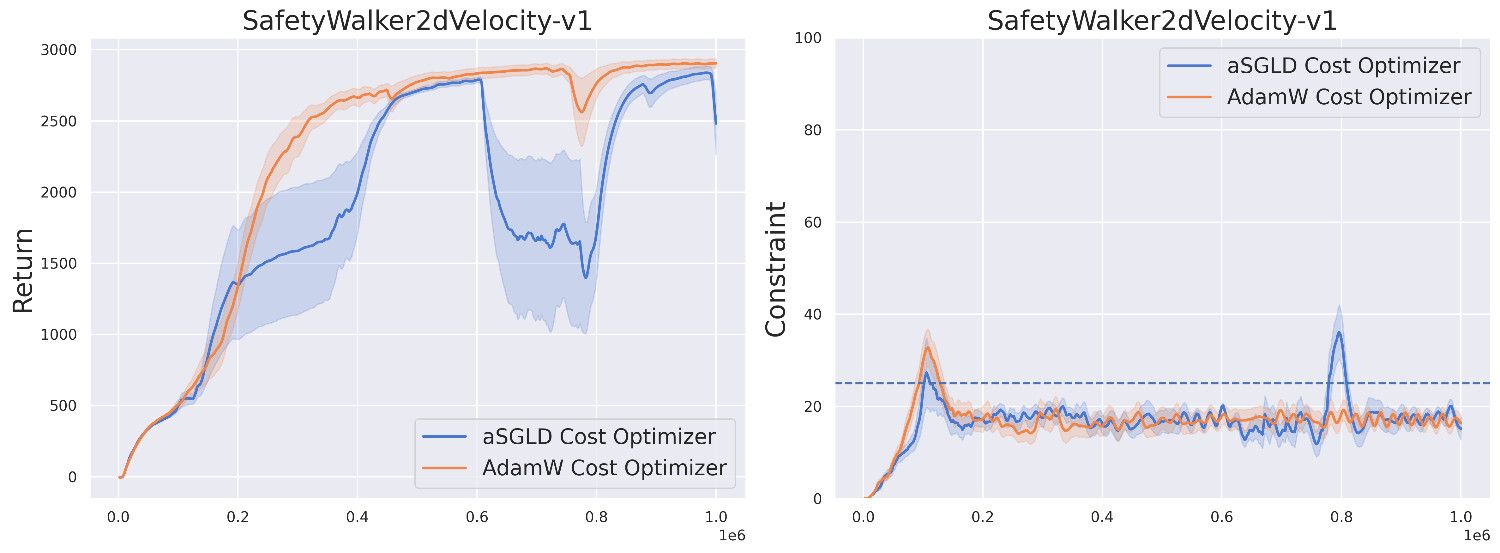}
    \end{minipage}
    \caption{Ablation of Cost Critic Optimizer (AdamW vs. aSGLD).}
    \label{fig:cost_optimizer_ablation}
\end{figure}

\subsection{Optimizer Selection for Distributional Cost Critic}
\label{app:cost_optimizer_ablation}

While we employ aSGLD for the reward critic ensemble, the distributional cost critic utilizes AdamW. This asymmetric design reflects their distinct functions: the reward critic drives exploration through posterior sampling, while the cost critic requires precise, stable risk estimation for reliable safety constraints. 

Figure~\ref{fig:cost_optimizer_ablation} presents training curves for three velocity environments. The results indicate that training the cost critic with aSGLD leads to less stable learning dynamics. In SafetyWalker2dVelocity-v1, the aSGLD variant exhibits increased volatility in returns and fails to maintain consistent safety performance in later training stages. Similarly, in SafetyHalfCheetahVelocity-v1, cost estimates show larger oscillations and slower convergence compared to the AdamW optimizer.

This instability likely stems from the distinct objectives of the two critics. The cost critic requires precise gradient updates to accurately regress IQN tail quantiles for CVaR, while aSGLD's parameter-space noise disrupts this process. Conversely, the reward critic leverages this stochasticity for posterior sampling to guide exploration. Therefore, we retain AdamW for stable cost estimation while reserving aSGLD for the reward critic.

\begin{table*}[t]
    \centering
    \caption{Hyperparameters}
    \label{tab:hyperparameters}
    \resizebox{\textwidth}{!}{%
    \begin{tabular}{lcccccc}
        \toprule
        \textbf{Hyperparameter} & \textbf{SACLag} & \textbf{TD3Lag} & \textbf{SACPID} & \textbf{TD3PID} & \textbf{WCSAC} & \textbf{SL-SAC (Ours)} \\
        \midrule
        Number of hidden layers & 2 & 2 & 2 & 2 & 2 & 2 \\
        Number of hidden nodes & 256 & 256 & 256 & 256 & 256 & 256 \\
        Activation function & ReLU & ReLU & ReLU & ReLU & ReLU & ReLU \\
        \midrule
        Batch Size & 256 & 256 & 256 & 256 & 256 & 256 \\
        Replay Buffer Size & $10^6$ & $10^6$ & $10^6$ & $10^6$ & $10^6$ & $10^6$ \\
        Discount factor ($\gamma$, reward) & 0.99 & 0.99 & 0.99 & 0.99 & 0.99 & 0.99 \\
        Discount factor ($\gamma_C$, cost) & 0.99 & 0.99 & 0.99 & 0.99 & 0.99 & 0.99 \\
        \midrule
        Optimizer & Adam & Adam & Adam & Adam & Adam & aSGLD \\
        Optimizer parameters ($\beta_1, \beta_2$) & (0.9, 0.999) & (0.9, 0.999) & (0.9, 0.999) & (0.9, 0.999) & (0.9, 0.999) & (0.9, 0.999) \\
        Inverse temperature ($T^{-1}$) & -- & -- & -- & -- & -- & $10^{-8}$ \\
        \midrule
        Number of twin reward critics & 1 & 1 & 1 & 1 & 1 & 3 \\
        Steps per epoch & 2000 & 2000 & 2000 & 2000 & 2000 & 2000 \\
        \midrule
        Quantile Embedding Dim & -- & -- & -- & -- & 64 & 64 \\
        Number of quantiles ($N$) & -- & -- & -- & -- & 32 & 32 \\
        CVaR risk level ($\epsilon$) & -- & -- & -- & -- & 0.5 & 0.5 \\
        \bottomrule
    \end{tabular}
    }
\end{table*}

\section{Experiment}
\label{sec:experimental_setup}

\subsection{Hyperparameters}
\label{subsec:hyperparameters}

All experiments were conducted on an NVIDIA V100 GPU with PyTorch 2.6 and CUDA 12.4.

To ensure stable Langevin dynamics without extensive hyperparameter tuning, we adopt the inverse temperature $T^{-1}$, bias factor $a$, and gradient clipping mechanism from LSAC~\citep{ishfaq2025langevin}. Specifically, we clip the combined gradient and adaptive bias term $\nabla_{\phi} L_{Q}(\phi) + a\zeta_{\phi}$ at threshold $c = 0.7$ to prevent gradient explosion during reward critic updates.

For the distributional cost critic, we use the Quantile Huber Loss with threshold $\kappa = 1$. The ensemble size $M$ and CVaR risk level $\epsilon$ were selected based on the ablation studies in Appendix~\ref{app:ablation_studies}. Standard RL hyperparameters (learning rates, buffer size, discount factors) follow the default configurations from the Omnisafe benchmark~\citep{ji2024omnisafe}.

Training proceeds in two warmup phases: a 5,000-step general warmup for data collection, followed by a 100,000-step multiplier warmup during which $\lambda$ remains fixed. After this period, $\lambda$ is updated at every timestep according to Equation~\ref{eq:lambda_update}.

\subsection{Environments}
Safety Gymnasium environments provide continuous control tasks where agents must balance task performance with safety constraints. Figure~\ref{fig:safety_mujoco_envs} illustrates the different environments used in our evaluation.

\textbf{Velocity Tasks.} In Velocity tasks (e.g., \texttt{AntVelocity}, \texttt{Walker2dVelocity}), agents maximize forward speed subject to a static velocity limit. A unit cost is incurred at each timestep where this threshold is exceeded. These environments simulate kinematic constraints in robotics, where adherence to operational speed limits is critical to prevent mechanical instability.

\textbf{Navigation Tasks.} We also evaluate on navigation tasks (e.g., \texttt{PointGoal}, \texttt{CarCircle}), where agents must navigate constrained 2D maps with static hazards. Circle tasks incur a unit cost on wall contact, while goal tasks impose a cost proportional to penetration depth when entering hazard regions. These tasks introduce complex spatial constraints and non-holonomic dynamics, analogous to collision avoidance in autonomous driving.

\begin{figure}[t]
    \centering
    \begin{subfigure}[b]{0.24\textwidth}
        \centering
        \includegraphics[width=\textwidth]{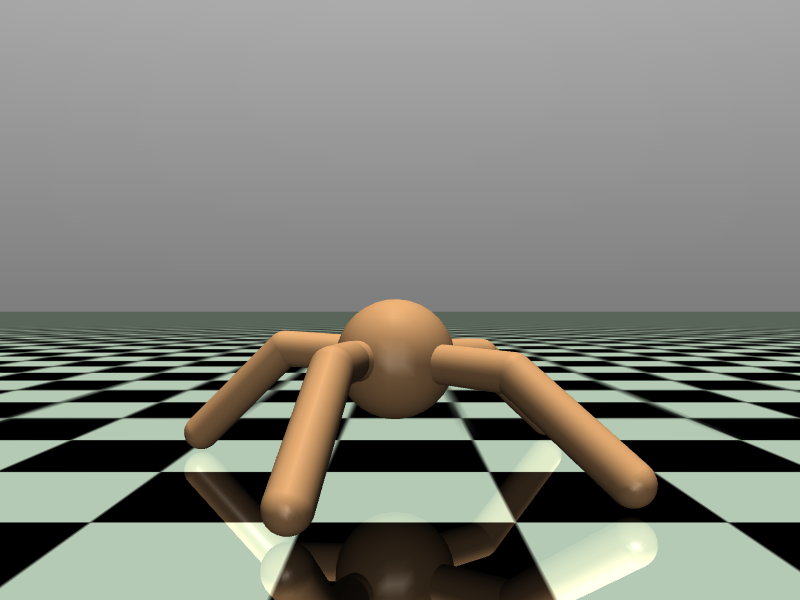}
        \caption{Ant}
        \label{fig:ant_env}
    \end{subfigure}
    \hfill
    \begin{subfigure}[b]{0.24\textwidth}
        \centering
        \includegraphics[width=\textwidth]{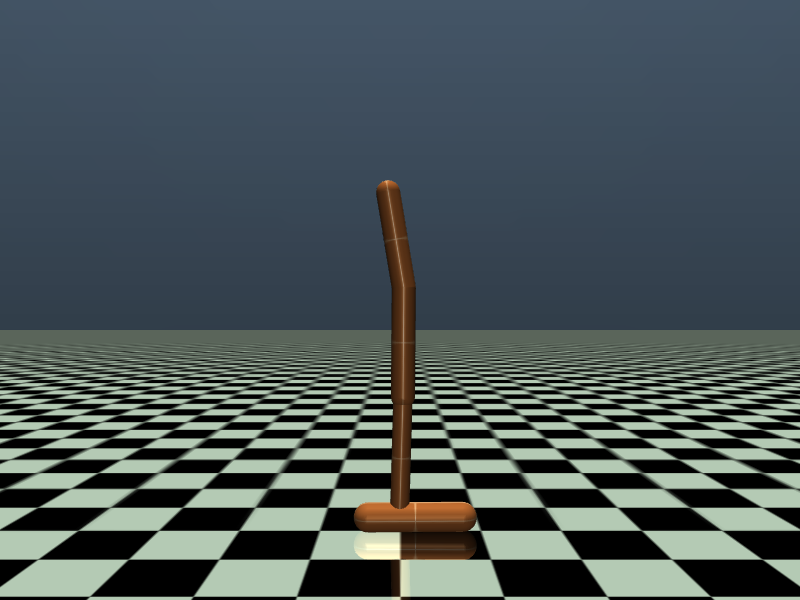}
        \caption{Hopper}
        \label{fig:hopper_env}
    \end{subfigure}
    \hfill
    \begin{subfigure}[b]{0.24\textwidth}
        \centering
        \includegraphics[width=\textwidth]{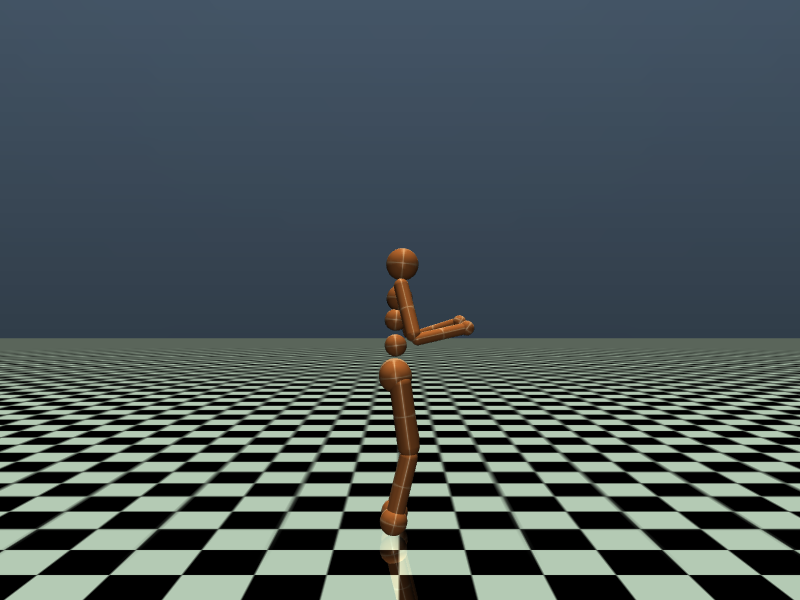}
        \caption{Humanoid}
        \label{fig:humanoid_env}
    \end{subfigure}
    \hfill
    \begin{subfigure}[b]{0.24\textwidth}
        \centering
        \includegraphics[width=\textwidth]{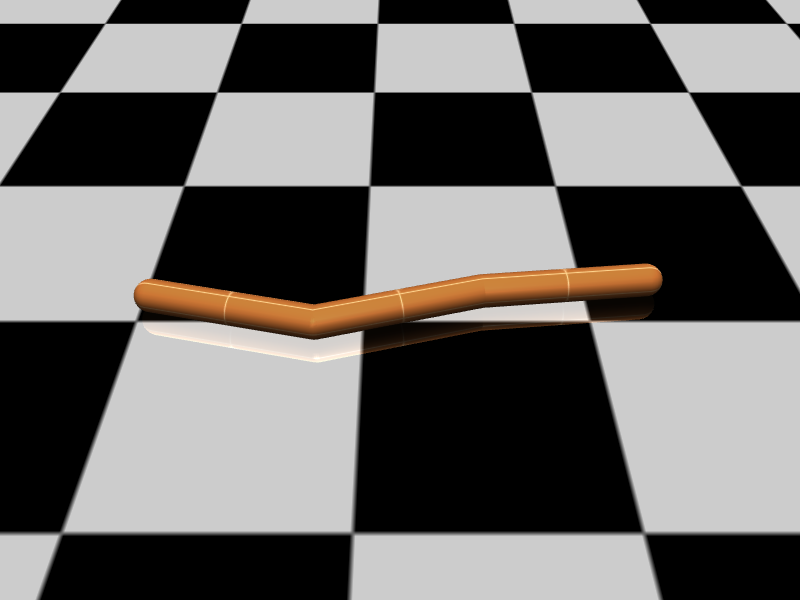}
        \caption{Swimmer}
        \label{fig:swimmer_env}
    \end{subfigure}
    \caption{Safety MuJoCo velocity-constrained locomotion environments.}
    \label{fig:safety_mujoco_envs}
\end{figure}

\begin{table*}[h]
\centering
\caption{Average Training time (hours) for 1M steps on Safety-Gymnasium tasks.}
\label{tab:training_time}
\begin{tabular}{lcccccc}
\toprule
\textbf{Environment} & \textbf{SACLag} & \textbf{SACPID} & \textbf{TD3Lag} & \textbf{TD3PID} & \textbf{WCSAC} & \textbf{SL-SAC} \\
\midrule
SafetyAntVelocity-v1 & 5.56 & 4.79 & 6.03 & 5.98 & 6.28 & 6.57 \\
SafetyWalker2dVelocity-v1 & 6.03 & 5.90 & 6.09 & 6.00 & 5.94 & 6.22 \\
\bottomrule
\end{tabular}
\end{table*}

\subsection{Training Time}
\label{subsec:computational_efficiency}

We measure the training time (in hours) for 1M environment steps on two representative Safety-Gymnasium tasks. All experiments were conducted on a single Nvidia V100 GPU, with times averaged over 5 seeds. Table~\ref{tab:training_time} reports the wall-clock training time for each algorithm.

SL-SAC introduces limited overhead which stems from the ensemble of reward critics and the aSGLD optimizer, which are essential for the improved safety and robust value estimation. 

\begin{figure*}[h] % The * spans both columns. [t] tries to place it at top of page.
    \centering
    
    \begin{subfigure}[b]{0.48\textwidth}
        \centering
        % Replace with your first environment image
        \includegraphics[width=\linewidth]{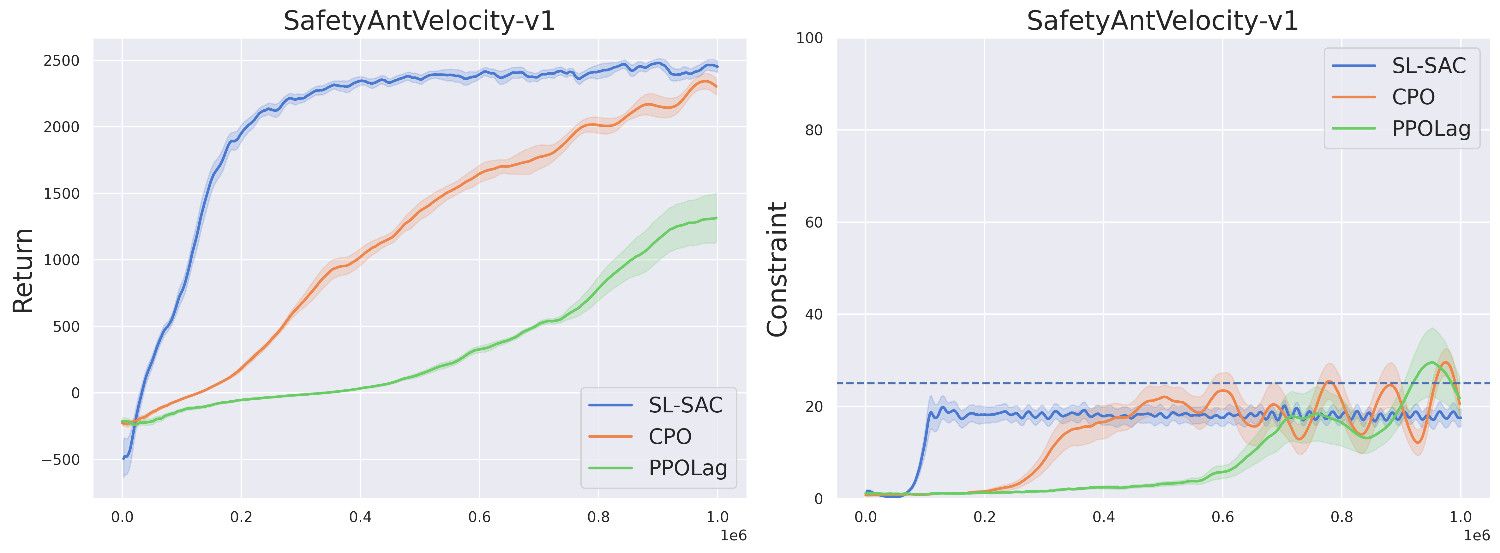} 
        \label{fig:ant_slsac_more_exps}
    \end{subfigure}
    \hfill % Adds flexible space between the two images
    \begin{subfigure}[b]{0.48\textwidth}
        \centering
        % Replace with your second environment image
        \includegraphics[width=\linewidth]{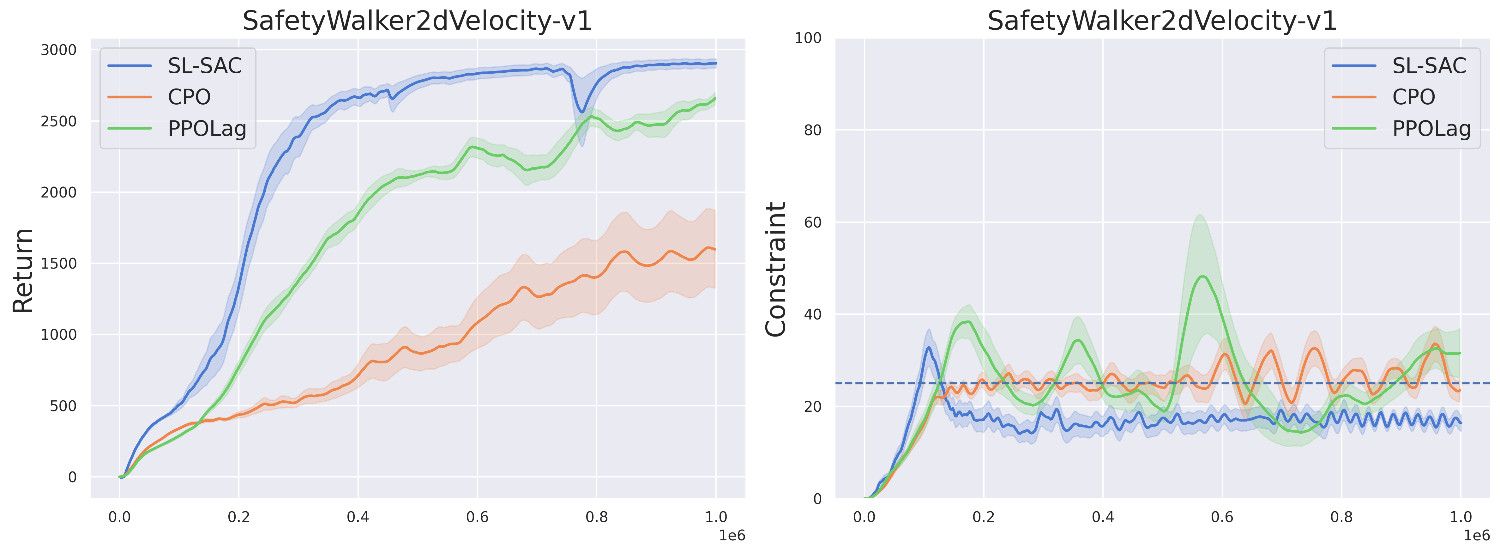}
        \label{fig:walker_slsac_more_exps}
    \end{subfigure}
    
    \caption{Training curves for additional baselines. Episode return (left) and cost (right), averaged over 5 seeds (solid line: mean, shaded area: standard deviation). The dashed line indicates the constraint threshold $\beta=25$.}
    \label{fig:additional_experiments}
\end{figure*}

\subsection{Comparison with On-Policy Baselines}
\label{subsec:on_policy_comparison}
To further validate the efficiency of SL-SAC, we compare it against prominent on-policy SafeRL algorithms: Constrained Policy Optimization (CPO)~\citep{achiam2017constrained} and PPO-Lagrangian (PPOLag)~\citep{ray2019benchmarking}. CPO is a trust-region method that analytically solves for policy updates within a feasible trust region to satisfy constraints at every step, often considered a theoretical benchmark for safety. PPOLag adapts the popular Proximal Policy Optimization (PPO) algorithm by incorporating a Lagrangian multiplier to penalize constraint violations dynamically.

The experimental results (Figure~\ref{fig:additional_experiments}) show a difference in sample efficiency between the approaches. SL-SAC reaches high task performance with fewer environment interactions than the on-policy baselines. In terms of safety, while CPO and PPOLag generally satisfy the constraints, they exhibit more variation around the cost threshold, whereas SL-SAC maintains a more consistent safety profile with lower variance over the learning curve.